\documentclass[a4paper]{article}
\usepackage[a4paper,top=3.5cm,bottom=3cm,left=3.5cm,right=3.5cm,marginparwidth=1.5cm]{geometry}
\usepackage{amsmath}
\usepackage{amsthm}
\usepackage{amssymb}
\usepackage[ruled]{algorithm}
\usepackage{algpseudocode}
\usepackage{bbm}
\usepackage{mathtools}

\usepackage{graphicx}
\usepackage{tcolorbox}
\usepackage{setspace}
\usepackage{makecell}
\usepackage{enumitem}
\usepackage{hyperref}      % hyperlinks
\usepackage{soul}
\usepackage[normalem]{ulem}

\usepackage{multirow}

\hypersetup{
    colorlinks=true,
    citecolor=blue,
    linkcolor=blue,
    filecolor=blue,      
    urlcolor=blue,
    pdftitle={Overleaf Example},
    pdfpagemode=FullScreen,
    }
\usepackage{url}            % simple URL typesetting
\usepackage{booktabs}       % professional-quality tables
\usepackage{amsfonts}       % blackboard math symbols
\usepackage{nicefrac}       % compact symbols for 1/2, etc.
\usepackage{microtype}      % microtypography
\usepackage{xcolor}         % colors
\usepackage{subcaption}
\usepackage{array}
\usepackage[english]{babel}
\usepackage[backend=biber, style=alphabetic, minalphanames=3]{biblatex}
\usepackage{tikz}

\newcommand{\cA}{\mathcal{A}}
\newcommand{\cO}{\ensuremath{\mathcal{O}}}
\newcommand{\cX}{\ensuremath{\mathcal{X}}}
\newcommand{\cY}{\ensuremath{\mathcal{Y}}}
\newcommand{\cZ}{\ensuremath{\mathcal{Z}}}

\newcommand{\cR}{\mathcal{R}}
\newcommand{\cQ}{\mathcal{Q}}

\newcommand{\cE}{\mathcal{E}}
\newcommand{\cC}{\mathcal{C}}

\newcommand{\cM}{\mathcal{M}}
\newcommand{\cN}{\mathcal{N}}

\newcommand{\cW}{\mathcal{W}}

\newcommand{\bE}{\mathbf{E}}

\newcommand{\re}{\mathbb{R}}
\newcommand{\E}{\mathbb{E}}

\newcommand{\norm}[1]{\left\|#1\right\|}

\DeclareMathOperator*{\argmax}{arg\,max}
\DeclareMathOperator*{\argmin}{arg\,min}

\DeclareMathOperator{\polylog}{polylog}

\theoremstyle{plain}
\newtheorem{thm}{\protect\theoremname}
\theoremstyle{definition}
\newtheorem{defn}[thm]{\protect\definitionname}
\theoremstyle{plain}
\newtheorem{cor}[thm]{\protect\corollaryname}
\theoremstyle{plain}
\newtheorem{lem}[thm]{\protect\lemmaname}
\theoremstyle{plain}

\theoremstyle{plain}

\makeatother

\providecommand{\corollaryname}{Corollary}
\providecommand{\definitionname}{Definition}
\providecommand{\lemmaname}{Lemma}
\providecommand{\theoremname}{Theorem}
\providecommand{\assumptionname}{Assumption}
\providecommand{\propositionname}{Proposition}

\newif\ifnotes
\notestrue
\ifnotes

\title{Differentially Private Worst-group Risk Minimization}
\author{Xinyu Zhou\thanks{Department of Computer Science \& Engineering, The Ohio State University, zhou.3542@buckeyemail.osu.edu} \and Raef Bassily\thanks{
Department of Computer Science \& Engineering and the Translational Data Analytics Institute (TDAI), The Ohio State University, bassily.1@osu.}\footnotemark[1]}

\begin{document}

\maketitle

\begin{abstract}

    We initiate a systematic study of worst-group risk minimization under $(\epsilon, \delta)$-differential privacy (DP). The  goal is to privately find a model that approximately minimizes the maximal risk across  $p$ sub-populations (groups) with different distributions, where each group distribution is accessed via a %\rnote{just ``sample oracle'' why insisting on ``stochastic''? it's obvious; what else it would be?} 
    sample oracle. We first present a new algorithm  that achieves excess worst-group population risk of $\Tilde{O}(\frac{p\sqrt{d}}{K\epsilon} + \sqrt{\frac{p}{K}})$, where $K$ is the total number of samples drawn from all groups and $d$ is the problem dimension. Our rate is nearly optimal when each distribution is observed via a fixed-size dataset of size $K/p$. %\rnote{we should consider adding a short sentence on optimality; e.g., ``Our rate is optimal when each distribution is observed via a fixed sample of size $K/p$''}\xnote{I added the suggested sentence.} 
    Our result is based on a new stability-based analysis for the generalization error. In particular, we show that $\Delta$-uniform argument stability implies $\Tilde{O}(\Delta + \frac{1}{\sqrt{n}})$ generalization error w.r.t. the worst-group risk, where $n$ is the number of samples drawn from each sample oracle. Next, we propose an algorithmic framework for worst-group population risk minimization using any DP online convex optimization algorithm as a subroutine. Hence, we give another excess risk bound of $\Tilde{O}\left( \sqrt{\frac{d^{1/2}}{\epsilon K}} +\sqrt{\frac{p}{K\epsilon^2}} \right)$. %\rnote{should add a sentence here to compare to the first bound}\xnote{Comparing these two bounds can be tricky since we need to cover multiple parameter regimes.} 
    Assuming the typical setting of $\epsilon=\Theta(1)$, this bound is more favorable than our first bound in a certain range of $p$ as a function of $K$ and $d$. Finally, we study differentially private worst-group \emph{empirical} risk minimization in the offline setting, where each group distribution is observed by a fixed-size dataset. We present a new algorithm with nearly optimal excess risk of $\Tilde{O}(\frac{p\sqrt{d}}{K\epsilon})$. %\rnote{I rephrased the last couple of sentences}
\end{abstract}

\section{Introduction}
Multi-distribution learning has gained increasing attention due to its close connection to robustness and fairness in machine learning. Consider $p$ groups (or, sub-populations), where each group is associated with some unknown data distribution. In the multi-distribution learning paradigm, the learner tries to find a model that minimizes the maximal population risk across all groups. Let $D_i$ denote distribution of group $i \in [p]$ and $\ell:\cW\times \cZ$ be a loss function over the parameter space $\cW$ and data domain $\cZ$, our objective of minimizing the worst-group population risk can be formulated as a minimax stochastic optimization problem written as
\begin{equation}
    \min_{w\in \cW} \max_{i\in [p]} \left\{L_{D_i}(w)\coloneqq\E_{z\sim D_i} \ell(w, z)\right\} \label{eq:worst-group loss in intro}
\end{equation}
%\rnote{In the above, you didn't define any of the main terms. Nothing is said about the the loss $\ell$ and its inputs $w, z$ or the parameter space $\cW$}\xnote{I added a definition of the loss function.}
When each group distribution $D_i$ is observed by a dataset $S_i$, we also define the worst-group empirical risk minimization problem as
\begin{equation}
    \min_{w\in \cW} \max_{i\in [p]} \left\{L_{S_i}(w)\coloneqq\frac{1}{|S_i|} \sum_{z\sim S_i} \ell(w, z)\right\} \label{eq:worst-group empirical loss in intro}
\end{equation}

The worst-group risk minimization problems  \eqref{eq:worst-group loss in intro} and \eqref{eq:worst-group empirical loss in intro} have wide applications in a variety of learning scenarios. In the regime of robust learning, the objective represents a class of problems named as group distributionally robust optimization (Group DRO) \cite{soma2022optimal, zhang2023stochastic, sagawa2019distributionally}. This formulation also captures the mismatch between distributions from different domains 
(or, sub-populations). %\rnote{the two previous sentences are quite vague: the source/target distributions and the referred to ``mismatch'' come out of the blue.}\xnote{Maybe we can say something like "discrepancies between the training and testing distributions."}.
In this type of scenarios, instead of assuming a fixed target distribution, this formulation can be used to find a model that can work well for any possible target distribution formed as a mixture of the distributions from  different domains/sub-populations. From the perspective of learning with fairness, each group may represent a protected class or a demographic group. Minimizing the worst-group risk leads to the notions of good-intent fairness \cite{mohri2019agnostic} or min-max group fairness \cite{abernethy2022active, diana2021minimax, papadaki2022minimax}. By making sure the worst-off group performs as good as possible, the objective prevents the learner from overfitting to certain groups at the cost of others. Moreover, the objectives are also connected to other learning applications such as collaborative learning \cite{blum2017collaborative} and agnostic federated learning \cite{mohri2019agnostic} where the output model is optimized to perform well across multiple distributions. 

Meanwhile, machine learning algorithms typically depend on large volumes of data, which may pose serious privacy risk, particularly in certain privacy-sensitive applications like healthcare and finance. Therefore, it is important that we provide a strong and provable privacy guarantee for such algorithms to ensure that the sensitive information always remains private during the learning process. Although there has been a significant progress in understanding worst-group risk minimization in the non-private setting \cite{soma2022optimal, haghtalab2022demand, zhang2023stochastic}, this problem has not been formally studied in the context of learning with provable privacy guarantees. 

Motivated by this, in this paper, we initiate a formal study of worst-group risk minimization under $(\epsilon, \delta)$-differential privacy (DP). We consider the setting where the loss function is convex and Lipschitz over a compact parameter space $\cW\subset \re^d$. For the sample access model, each group distribution is accessed via a sample oracle and the learner is able to sample new data points from any group distribution via its sample oracle as needed during the learning process.

\subsection{Contribution}
As we mentioned earlier, we let $p$ denote the number of groups , 
$K$ denote the total number of samples drawn from all groups and $d$ denote the problem dimension. %Recall that $K$ is split evenly across all groups in the offline setting, i.e., $K=n\cdot p$ where $n$ is the size of the dataset obtained from each group. In the online setting, we do not have this constraint. %the total sample size $K$ is equal to the sum of the dataset sizes across all groups in the offline setting and is equal to the total number of oracle calls in the stochastic oracle setting. \rnote{This paragraph should come earlier in the intro. At least, you can use it to introduce the main parameters, e.g., $d, p, K, \epsilon, \delta$ etc.}\xnote{I move the definition of the main parameters to the paragraph before contributions.}

We first provide a new algorithm adapted from the phased-ERM approach in \cite{feldman2020private} with the excess worst-group population risk of  $\Tilde{O}\left(\sqrt{\frac{p}{K}} +\frac{p\sqrt{d}}{K\epsilon}\right)$. %\rnote{We should briefly comment on optimality; see my comment in the abstract}\xnote{I added a sentence on the optimality. Also, do we need to mention the rate matches the non-private lower bound when $d = O(K/p)$?} 
The first term in this bound matches the optimal non-private bound \cite{soma2022optimal} and the second term represents the cost of privacy. Our rate is optimal up to logarithmic factors in the offline setting where each distribution is observed via a a fixed-size dataset of size $K/p$.  The utility guarantee of our algorithm relies on a new stability based argument to establish the generalization error bound. Roughly speaking, we show that $\Delta$-\emph{uniform argument stability} %\rnote{We rely on argument uniform stability rather than uniform stability, correct? I would add the term ``uniform'' here and refer to the definition.}\xnote{Yes, it is supposed to be argument uniform stability} 
(see Definition~\ref{def:stable}) with respect to the parameter $w$ %\rnote{the term ``primal'' is not introduced before. Actually, there is no need to call it primal parameter here since this is the only parameter vector introduced in ~(\ref{eq:worst-group loss in intro}). }\xnote{removed ``primal``.} 
implies a $\Tilde{O}(\Delta + 1/\sqrt{n})$ generalization error where $n$ is the number of samples drawn from each sample oracle. This is distinct from existing generalization results based on uniform convergence \cite{mohri2019agnostic, abernethy2022active}, which lead to suboptimal rates for general convex losses. On the other hand, our result also circumvents a known bottleneck in deriving generalization guarantees using argument uniform stability  in the stochastic saddle-point problems. In particular, in the $L_2/L_2$ setting, the best known generalization guarantee for a $\Delta$-uniform argument stable algorithm is $\sqrt{\Delta}$ generalization error based on the analysis of the strong duality gap \cite{ozdaglar2022good}. %\rnote{We MUST add references for this last statement.} \rnote{Important: We should mention and briefly describe our phased ERM algorithm: perhaps, give a shorter/more concise version of our description in Section~\ref{sec: population offline}. Highlight its distinction from that of Feldman et al.}\xnote{I added the reference and several sentences to describe our algorithm.} 

Our algorithm is akin to the Phased ERM approach of \cite{feldman2020private}, which was introduced to solve private stochastic convex optimization. %\rnote{I edited the previous sentence} 
We repurpose this approach  for the private stochastic minimax optimization. In particular, we define a sequence of regularized minimax problems, which are solved iteratively using any generic non-private solver. Our version of this algorithm requires different tuning of parameters across iterations and, more crucially, a new analysis that utilizes the uniform argument stability of the solutions of the regularized minimax problems.

As a side product, our stability result naturally leads to a new non-private regularization based method for the worst-group risk minimization problem, %\rnote{do we need *iterative* regularization in the non-private setting or just regularization?}\xnote{We only need regularization to achieve the optimal rate in the non-private setting.} 
which matches the non-private lower bound in \cite{soma2022optimal} up to logarithmic factors. This may be of independent interest because existing non-private methods with (nearly) optimal rate for this problem rely on single-pass first-order methods rather than regularization. %\textcolor{red}{We note that there is a recent private construction \cite{bassily2023differentially} based on recursive regularization for the stochastic saddle point problem, however this construction and its analysis are specific to the $L_2/L_2$ setting, }

%\rnote{You need to say something about BGM 2023 like I discussed with you in the email after Cristobal's comment. You can mention that their method is also based on recursive regularization for the stochastic saddle point problem but their construction and analysis are devised for the L2/L2 case and they do not seem to work for the worst-group risk minimization problem.} %on the sampled data points %\rnote{please, always write ``datapoints'' as ``data points''} 
%to directly optimize the population risk while {\color{red} our approach is based on iterative regularization.}\xnote{The approach is based on just regularization.} \sout{based method does not require such one-pass structure.}
 
Next, we give a framework %\rnote{Again, don't refer to the algorithm number. Just explain that we describe an algorithmic framework that requires a black-box access to any differentially private online convex optimization (DP-OCO) algorithm as a subroutine.} 
to minimize the worst-group population risk using a black-box access to any DP online convex optimization (DP OCO) algorithm as a subroutine. Our framework leverages the idea in \cite{soma2022optimal, haghtalab2022demand} on casting the worst-group risk minimization objective in a zero-sum game and decomposing the excess worst-group risk into the expected regret of both players. By instantiating our framework with the DP-FTRL algorithm from \cite{kairouz2021practical}, we obtain another bound on the excess worst-group population risk %\rnote{You never introduced the term ``optimality gap'' before. Why not try to be consistent and use ``expected worst-group risk''?}\xnote{I changed the term} 
of $\Tilde{O}\left( \sqrt{\frac{d^{1/2}}{\epsilon K}} +\sqrt{\frac{p}{K\epsilon^2}} \right)$. Assuming the typical setting for the privacy parameter $\epsilon$, namely, $\epsilon = \Theta(1)$, this bound is more favorable than our bound based on the Phased ERM approach in a certain range of $p$ as a function of $K$ and $d$, particularly, when $p \geq \max\{\sqrt{d}, K/d\}$ or $\sqrt{\frac{K}{d^{1/2}}}\leq p\leq \sqrt{d}$. %\xnote{I make some changes and use the inequality $\sqrt{\frac{d^{1/2}}{K}} \leq \sqrt{\frac{p}{K}}\leq \frac{p\sqrt{d}}{K}$ to derive this condition. I also tried to stress this regime is just an example} \rnote{I made few edits here. What about the other inequality:  $\sqrt{\frac{p}{K}}\leq \sqrt{\frac{d^{1/2}}{K}} \leq \frac{p\sqrt{d}}{K}$. Perhaps, we should state the condition resulting from this inequality as well. I.e., say when $p \geq ...$ or $p \geq ...$} \xnote{This inequality gives $\sqrt{\frac{K}{d^{1/2}}}\leq p\leq \sqrt{d}$ and we also need to make sure this range is not empty by setting some constraints on $K$ and $d$.} %This rate matches the non-private lower bound up to logarithmic factors in the regime of $d = O(p^2)$.
%\rnote{We must compare this to our first bound. I.e., which is better in what parameter regimes, and comment on optimality.}\xnote{I added the optimality comment. But comparing it with the first approach might be tricky since there are too many cases to cover. (Both match the non-private lower bound, one does and one does not, both don't...) }\xnote{Note this bound can be tighter than our first bound based on Phased ERM approach when ....}
%For the common case of $\epsilon = \Theta(1)$, our rate matches the lower bound of $O\left( \sqrt{\frac{p}{K}} \right)$ from \cite{soma2022optimal} up to logarithmic factors when $d = \Tilde{O}\left(p^2\right)$. This condition can be naturally satisfied in many practical scenarios where the number of groups is relatively large compared to $\sqrt{d}$. 

%\rnote{I rephrased the first few sentences in this paragraph} 
Finally, we investigate the worst-group \emph{empirical} risk minimization problem \eqref{eq:worst-group empirical loss in intro} %\rnote{remind the reader by referring to the problem you defined in~(\ref{eq:worst-group empirical loss in intro})} 
 in the \emph{offline setting}, where each distribution $D_i$ is observed by a fixed-size dataset $S_i$. %\rnote{there is no point of saying that since here we only give empirical risk guarantees not population risk} 
We present a new algorithm based on a private version of the multiplicative group reweighing scheme \cite{abernethy2022active, diana2021minimax}. We show that our algorithm achieves a nearly optimal excess worst-group empirical risk of $\Tilde{O}\left(\frac{p\sqrt{d}}{K\epsilon}\right)$.  %\xnote{I am not sure how to present the optimality of Phased ERM algorithm in the offline setting. Do we mention it after we introduce the offline setting or when we present the algorithm (second paragraph of section 1.1?)}\rnote{See my earlier comment in the second paragraph of section 1.1. I think we should mention it there, i.e., as soon as we present the rate achieved by phased ERM. It's totally ok to be a little redundant and mention the offline setting twice; once after we mention the rate of phased ERM (albiet briefly) and again here in this paragraph in the context of empirical risk minimization and state the optimality of our upper bound on the empirical risk.}\xnote{I added one sentence on the optimality of Phased ERM in the second paragraph of section 1.1.}

%\xnote{Add a table to summarize our result and compare with existing non-private results?}\rnote{I recommend adding a table for the population risk results only for both offline and online settings (the table should include our bounds for the private case as well as the  optimal bounds known in the non-private case (and cite the relevant works). We should show the non-private bounds as $\Theta(\cdot)$ since they are optimal. } \xnote{The non-private bound in the offline setting is lower bounded by $\Omega(1/\sqrt{n})$, so current non-private rate is only optimal up to logarithmic factors. Also, \cite{haghtalab2022demand} achieves a rate of $O\left(\sqrt{\frac{p\ln p}{K}}\right)$. Do we also need to mention this rate in the table?}

%\rnote{should comment on this saying that this condition is naturally satisfied in many practical scenarios where the number of groups is relatively large compared to $\sqrt{d}$.}\xnote{added the comment.}
%\rnote{I think you should elaborate more on the technical contributions and the ideas of the algorithms and their relation to the existing non-private ones. You didn't mention anything about  \cite{haghtalab2022demand} and their technique. If you don't think there is much more to be said, then we should not write our contributions as a subsection but rather as a titled paragraph. }\xnote{I add one sentence on the connection to related works.}

\subsection{Related Work}
The worst-group risk minimization is closely related to the group DRO problem. In \cite{sagawa2019distributionally}, the authors consider solving the empirical objective in the offline setting and propose an online algorithm based on the stochastic mirror descent (SMD) method from \cite{nemirovski2009robust}. The convergence rate achieved in \cite{sagawa2019distributionally} is suboptimal because of the high variance of the gradient estimators. On the other hands, several works \cite{haghtalab2022demand, soma2022optimal, zhang2023stochastic} has studied the group DRO problem in the stochastic oracle setting and achieves an excess worst-group population %\rnote{``empirical''?! I think this is still stochastic setting but simply online. Also, don't Haghtalab et al. consider the population risk as their metric in the convex loss setting? The term empirical here does not make sense to me and is quite confusing} 
risk of $\Tilde{O}\left(\sqrt{\frac{p}{K}}\right)$
%\xnote{It is a typo. They consider the population risk. I changed it to "population risk".} %\rnote{Again, you never introduced the term ``optimality gap'' before. Please, try to be consistent and use ``expected worst-group risk''?}\xnote{I changed the term.}  \rnote{The constant $G, M, L$ are not defined yet. Please, remove them from that bound.}\xnote{I removed the constant terms}. 
The basic idea behind the methods used in these works is to cast the objective into a zero-sum game based on the observation that the optimality gap depends on the expected regret of the max and min players. The authors in \cite{soma2022optimal} demonstrate the tightness of this rate by proving a matching lower bound result. %\rnote{isn't Haghtalab et al. appear more recent (and contain better upper bounds and matching lower bound? Why mention optimality of Soma et al. but not Haghtalab et al. here?}\xnote{Haghtalab et al. doesn't give the lower bound for general convex loss class. They only give a lower bound for finite hypothesis class.}

The worst-group risk minimization problem is also extensively studied under the notion of min-max group fairness. Reference \cite{diana2021minimax} studies learning with min-max group fairness and propose a multiplicative reweighting based method. However, their method assumes the access to a weighted empirical risk minimization oracle which may not be realistic in practice. Meanwhile, reference \cite{abernethy2022active} studies a similar problem and present more efficient gradient based methods to achieve the min-max group fairness based on the idea of active sampling. The problem of min-max group fairness is also studied in \cite{mohri2019agnostic} under the federated learning setting. However, all these works do not provide any privacy guarantees.

Our work is also related to the line of works on differentially private stochastic saddle point (DP-SSP) problem. References \cite{yang2022differentially, zhang2022bring} have presented algorithms with the optimal rate for the weak duality gap, while the optimal rate for the strong duality gap is achieved in \cite{bassily2023differentially}. However, these works only consider the objective with $L_2/L_2$ geometry, directly applying the results from existing DP-SSP works leads to a suboptimal convergence rate due to the polynomial dependence of $p$ on the $L_2$ Lipschtiz constant for the model parameters. Moreover, extending the recursive regularization method in \cite{bassily2023differentially} to the $L_2/L_1$ setting is not straightforward because the given analysis explicitly requires both primal and dual parameters lie in Hilbert spaces. In a concurrent and independent work \cite{gonzalez2022optimal}, the authors study DP-SSP problem in the $L_1/L_1$ geometry over polyhedral feasible sets. They propose a method based on a combination of stochastic mirror descent and the exponential mechanism, and show that it attains nearly dimension-independent convergence rate for the expected duality gap. We note that their problem setting is different from the $L_2/L_1$ setting considered in this work.  In particular, their method takes advantage of the polyhedral geometry under $L_1$ constraints. Consequently, their construction and results are incomparable to ours.

%There is a concurrent and independent work \cite{gonzalez2022optimal} studying DP-SSP under $L_1/L_1$ geometry over polyhedral feasible sets and attaining nearly dimension independent convergence rate for the expected duality gap. Their setting is different from the $L_2/L_1$ setting considered in this work. More crucially, their results are based on different method and analysis from ours. In particular, the method in \cite{gonzalez2022optimal} is built upon stochastic mirror decent and the privacy is guaranteed using exponential mechanism.}
%\rnote{The related work is quite long. You should try to be more concise. Also, as it is coming after a very short contributions section, it does not look good.}\xnote{We add new contents to the contribution so the length of the related work should be alright.}

\section{Preliminaries}\label{sec:prelims} 
%\rnote{it's better to call it ``Preliminaries''}}
%\rnote{In this section, we should also describe the sample oracles for the $p$ groups, etc.}\xnote{I will give a definition of the sample oracle after this paragraph.} 
We consider $p$ groups, where each group $i\in [p]$ corresponds to a distributions $D_i$. Given a parameter space $\mathcal{W}\in \cR^d$ and a loss function $\ell$, we let $\ell(w, z)$ be the loss incurred by a weight vector $w$ on a data point $z$. We assume the $L_2$ diameter of $\cW$, $\norm{\cW}_2$, is bounded by $M$. Throughout this paper, we assume that the loss function $\ell$ is convex, $L$-Lipschitz and bounded by $[0, B]$ unless stated otherwise .
%\rnote{This is not a standard notation. You should use $L$ for the Lipschitz constant, and use something else, e.g. $B$, for the bound on the loss. Also, since the $\cW$ is bounded by $M$ and the loss is $L$-Lipschitz, it's standard to assume that the loss is bounded by $ML$. You can make this assumption and write everything in terms of the Lipschitz constant $L$ and the diameter $M$. Please, make those changes every where.}\xnote{Assigning a specific notation  to the loss bound  is common in prior works like \cite{haghtalab2022demand, soma2022optimal}. I changed the notations of Lipschitz constant and parameter space diameter, and use $B$ for the loss bound}. 
The population loss of distribution $D_i$ is written as $L_{D_i}(w) = \E_{z\sim D_i} \ell(w, z)$. Given a dataset $S_i$, we also denote $L_{S_i}(w)=\frac{1}{|S_i|}\sum_{z\in S_i} \ell(w, z)$ as the empirical loss evaluated on $S_i$.  

\paragraph{Sample oracle:} A sample oracle $\cC_i$ w.r.t distribution $D_i,$ $i\in [p]$ returns an i.i.d sample $z$ drawn from $D_i$. We denote the collection of sample oracles from all groups as $\tilde{\cC} = \{\cC_1, \dots \cC_p\}$.%\rnote{Use a different symbol other than $O$ and $\tilde{O}$ so that it's not confused with the big-O notation!}\xnote{I changed it to $\cC$.}

\paragraph{$\alpha$-saddle point:} We say $(\tilde{x}, \tilde{y})$ is an $\alpha$-saddle point of a minimax problem $\min_{x\in \cX} \max_{y\in \cY} \phi(x, y)$ if $\phi(\tilde{x}, y) - \phi(x, \tilde{y}) \leq \alpha$ for any $x\in \cX$ and $y\in \cY$. When $\alpha=0$, we simply call $(\tilde{x}, \tilde{y})$ a saddle point of $\phi$.

\paragraph{Worst-group population risk minimization}: Given $p$ groups that each of them is with a distribution $D_i$, we define the expected (population) worst-group risk of model $w$ as the maximal population loss of $w$ across all group distributions written as
\begin{equation*}
    R(w) = \max_{i\in [p]} L_{D_i}(w)
\end{equation*}

Our objective of minimizing the worst-group risk is therefore written as
\begin{equation}
    \min_{w\in \mathcal{W}}\max_{i\in [p]} L_{D_i}(w) \label{eq:agnostic loss}
\end{equation}

Denote $\Delta_p = \{\lambda \in [0,1]^p : \norm{\lambda}_1 = 1\}$ as the probability simplex over $p$ dimensions and $\phi(w, \lambda) = \sum_{i=1}^p \lambda_i L_{D_i}(w)$.

Then the objective in \eqref{eq:agnostic loss} can be equivalently written as 
\begin{equation}
    \min_{w\in \cW} \max_{\lambda\in \Delta_p} \phi(w, \lambda)  \label{eq:agnostic loss 2}
\end{equation}
The equivalence is based on the fact that for any $w\in \cW$, $\phi(w, \lambda)$ is maximized by a $\lambda$ with all dimensions being zero except the one with the highest $L_{D_i}(w)$. We also define the \emph{excess worst-group population risk} of $w\in \cW$ as
\begin{equation*}
    \cE(w; \{D_i\}_{i=1}^p)\triangleq \max_{i\in [p]} L_{D_i}(w) - \min_{\tilde{w}\in \cW} \max_{i\in [p]} L_{D_i}(\tilde{w})
\end{equation*}
%\rnote{made minor edit in the notation of the eq. above to reduce confusion (added tilde over $w$ in the min term)}

\paragraph{Differential Privacy \cite{dwork2006calibrating}:} A mechanism $\cM$ is $(\epsilon, \delta)$-DP if for any two neighboring datasets $S$ and $S'$ that differ on single \emph{data point} and any output set $\cO$, we have
\begin{equation*}
    P\left(M(S)\in \cO\right) \leq e^\epsilon P \left(M(S')\in \cO\right) + \delta
\end{equation*}

%\sout{We focus on record-level privacy throughout the paper. In offline setting, we consider any two neighboring collection of datasets $\Tilde{S} = \{S_1, \dots S_p\}$ and $\Tilde{S}' = \{S'_1, \dots S'_p\}$ when $S_i = S'_i$ for any $i\neq j$ with some $j\in [p]$ and $S_j$ differs from $S'_j$ with a single datapoint. In online setting, we consider any two neighboring sequence of sampled datapoint $\{z_1, \dots z_T\}$ and $\{z_1', \dots z'_T\}$ such that $z_i = z'_i$ for any $i\neq j$ with some $j\in [p]$.} \rnote{The phrasing was confusing and order of quantifiers was wrong. Fixed this paragraph. See below.}

We focus on record-level privacy throughout the paper. We say that a pair of sequences of sampled data points $\{z_1, \dots z_T\}$ and $\{z_1', \dots z'_T\}$ are neighbors if, for some $j\in [p]$, $z_i = z'_i$ for all $i\in [T]\setminus\{j\}$ and $z_j \neq z'_j$.

%In offline setting, we say that $\Tilde{S} = \{S_1, \dots S_p\}$ and $\Tilde{S}' = \{S'_1, \dots S'_p\}$ are two neighboring collections of datasets if, for some $j\in [p]$, $S_i = S'_i$ for all $i\in [p]\setminus\{j\}$ and $S_j$ differs from $S'_j$ in a single data point. In online setting, 

\paragraph{Stability:} Our analysis depends on the notion of uniform argument stability defined as follows
\begin{defn}\label{def:stable}
    A randomized algorithm $\cA$ that takes data collections $\tilde{S}$ as input and outputs $(\cA_w(\tilde{S}), \cA_\lambda(\tilde{S}))$ is said to attain $\Delta$-uniform argument stability with respect to $w$ if for any pairs of adjacent $\tilde{S}$ and $\tilde{S}'$ that differs  in a %\rnote{``in'' not ``with''; saw this mistake a few times} 
    single data point, we have
    \begin{equation*}
        \E_\cA \left[\norm{\cA_w(\tilde{S}) - \cA_w(\tilde{S'})}\right]\leq \Delta
    \end{equation*}
\end{defn}
%\rnote{It's odd that all definitions in this section don't involve the defn environment except the last one.}

\section{Minimax Phased ERM} \label{sec:minimax phased ERM} 

In this section, we propose an algorithm attaining an excess worst-group population risk of $\tilde{O}\left(\sqrt{\frac{p}{K}} + \frac{p\sqrt{d}}{K\epsilon}\right)$, where $p$ and $K$ denote the number of groups and the total number of samples drawn from all groups, respectively. Our algorithm involves iteratively solving a sequence of regularized minimax objectives. %\rnote{ this sounds like the algorithm of Feldman et al. already involves a sequence of minimax problems! I am not sure if this is what we wrote here earlier. Please, rephrase. Perhaps, you don't need to refer to Feldman et al. here since we do this in more details in sec 3.1}\xnote{I removed the reference to Feldman et al.}. 
In iteration $t$, given the previous iterate $w_{t-1}$ and a dataset collection $\{S_i^t\}_{i=1}^p$ where $S_i^t$ is the dataset formed by the data points drawn from the sample oracle $\cC_i$ at current iteration, we find an approximate saddle point $(\tilde{w}_t, \tilde{\lambda}_t)$ of the regularized objective 
$F_t(w, \lambda) = \sum_{i=1}^p \lambda_i L_{S^t_i}(w) + \mu_w^t \norm{w-w_{t-1}}_2^2 - \mu_\lambda^t \sum_{j=1}^p \lambda_j \log \lambda_j $ %\rnote{the previous equation should be in the same line (or at least the same page).}
and output  $\tilde{w}_t$. We can show  that  $\tilde{w}_t$ has low sensitivity w.r.t. the collection of samples drawn from all groups. %\rnote{Should clarify what we mean by datasets here since we are no longer operating in the offline setting. The wording here should be tied to the way we present the model in the ``Preliminaries'' section. We can clarify things without too many words; e.g., say the collection of samples drawn from all groups}\xnote{I made the change as suggested} \rnote{You still introduce the regularized objective (which depends on the datasets before you introduce the datasets!}\xnote{I added a definition of $S_j^t$ here.}
, and hence, we can use the standard Gaussian mechanism to to ensure that the computation of $\tilde{w}_t$ is differentially private. 

Proving the utility guarantee of our algorithm involves new stability based argument. First, given a dataset collection $\tilde{S}=\{S_1, \dots S_p\}$ where each $S_i\sim D_i^n$ is a dataset formed by $n$ samples drawn from the sample oracle $\cC_i$. %\rnote{Change the symbol for the oracle as I mentioned earlier} %\rnote{What is $n$?! We agreed that such a fixed size datasets are no longer needed. Everything should be in terms of $K$ and $p$. If you need to refer to a generic collection of datasets, then at least explain that this collection of datasets is drawn using the sample oracles (which should have been defined in Section~\ref{sec:prelims})}\xnote{I defined the sample oracles in Section 2 and made the changes here.} 
we show that an algorithm that outputs an approximate saddle point $(\tilde{w}, \tilde{\lambda})$ for the regularized objective 
%\rnote{you didn't define what is meant by primal and dual; this seems to jump out of no where. Rephrase along the following lines ``... an algorithm that outputs a saddle point $(\tilde{w}, \tilde{\lambda})$ for the regularized objective... has argument stability ... w.r.t. $\tilde{w}$''}\xnote{I made the change based on the comment.}
$F(w, \lambda) = \sum_{i=1}^p \lambda_i L_{S_i}(w) + \frac{\mu_w}{2} \norm{w- w'}_2^2 - \mu_\lambda \sum_{j=1}^p \lambda_j \log \lambda_j$
has %\rnote{you mean *uniform* argument stability?}\xnote{I added the definition at the end of Section 2.} 
uniform argument stability 
%\rnote{You need to add a definition for uniform argument stability in the preliminaries section}
of $O\left(\frac{L}{n\mu_w} + \frac{B}{n\sqrt{\mu_w \mu_\lambda}}\right)$ with respect to $\Tilde{w}$. Then we prove that $\Delta$-uniform argument stability of $\tilde{w}$ leads to a $\tilde{O}\left(\Delta + \frac{B}{\sqrt{n}}\right)$ generalization error bound. 

Our stability-based generalization argument is distinct from existing generalization analysis based on uniform convergence \cite{mohri2019agnostic, abernethy2022active}, which leads to sub-optimal excess worst-group risk for general convex losses even in the non-private regime. On the other hand, our result circumvents a known bottleneck in the relationship between stability and generalization in stochastic minmax optimization (also, known as Stochastic Saddle Point (SSP) problems). In particular, in the $L_2/L_2$ setting of the SSP problem, the best known generalization guarantee using a $\Delta$-argument stable algorithm is $\sqrt{\Delta}$ generalization error based on the analysis of the strong duality gap \cite{ozdaglar2022good}.

\subsection{Algorithm description}
Our algorithm relies on a non-private subroutine to solve a regularized minimax problem. More concretely, let $\tilde{S}=\{S_1,\dots S_p \}$ be a collection of datasets. Let $\mu_w$, $\mu_\lambda >0$ and $w'\in \cW$. Let $\cA_{emp}(\tilde{S}, \mu_w, \mu_\lambda, w', \alpha)$ be an empirical minimax solver that computes an $\alpha$-saddle point $(\Tilde{w}, \Tilde{\lambda})$ of %\xnote{Do we need to highlight the definition of $\cA_{emp}$? It is far away from the algorithm.}\rnote{I am not sure what you mean. The algorithm is described in the very next paragraphs. Where else you want to describe $\cA_{emp}$? Obviously, we need to say what it is when it's introduced.}\xnote{I was thinking when people are looking at Algorithm 2, they may find it hard to locate the definition of the $\cA_{emp}$, which is far away from Algorithm 2 and not highlighted. }
\begin{equation} 
    \min_{w\in \cW} \max_{\lambda\in \Delta_p} \Bigg\{F(w, \lambda) \label{eq:regularized objective}\coloneqq \sum_{i=1}^p \lambda_i L_{S_i}(w) + \frac{\mu_w}{2} \norm{w- w'}_2^2 - \mu_\lambda \sum_{j=1}^p \lambda_j \log \lambda_j \Bigg\}
\end{equation}
and outputs $\tilde{w}$. %\rnote{made necessary edit in the eq. above. The curly brackets are necessary here since the RHS is not the min max; it's merely the objective}

The formal description of our algorithm is given in Algorithm \ref{alg:Minimax Phased ERM}. Our algorithm is inspired by the Phased ERM approach of \cite{feldman2020private}, which was used to attain the optimal rate for the simpler problem of differentially private stochastic convex optimization. We repurpose the Phased ERM approach for the stochastic minimax problem. Our algorithm involves several modifications to this approach that enable us to attain strong guarantees for the worst-group population risk minimization problem. %\sout{Our algorithm is based on a non-trivial modification of the Phased ERM approach that enables us to solve the worst-group population risk minimization problem.} %\sout{Crucially, we provide a new analysis for our modified algorithm showing that it attains a nearly optimal rate for this new problem. }\rnote{Perhaps, we should briefly re-mention the stability based analysis and refer to the lemma/theorem where the crucial stability-related result is given.}\xnote{I slightly change the statement given we are not emphasizing the optimality any more.} 
Crucially, we provide a new analysis for our algorithm that utilizes the argument uniform stability of the regularized minimax problems (see, Lemma \ref{lem: expected excess loss of hat_w}). %\rnote{I rephrased the above paragraph}

In Algorithm~\ref{alg:Minimax Phased ERM}, a sequence of  regularized empirical minimax problems are defined and solved iteratively using the non-private minimax solver $\cA_{emp}$ (defined earlier in this section). The total number of such problems (the number of rounds in the algorithm) $T$ is logarithmic in $K/p$. %\rnote{You need to remind the reader of $K$ and $p$. This is the first time $K$ is mentioned since the introduction! You need to reiterate what $K$ and $p$ are earlier in this section, e.g., after the first line in sec 3 (after you state the bound, say ``where $K$ and $p$ denote ..., respectively''}\xnote{I added the definitions of $K$ and $p$ in the beginning of Sec 3.}
In round $t\in [T]$, we first sample a dataset of size $n_t$ from each distribution to construct a new dataset collection $\tilde{S}_t$. The center of the regularization term is chosen to be $w' = w_{t-1}$, i.e., the previous iterate. The settings of the regularization parameters in each round are carefully chosen to ensure convergence to a nearly optimal rate. The regularization parameter $\mu_\lambda$ is fixed across all rounds whereas $\mu_w$ is doubled with every round. %Note that the regularization parameter $\mu_\lambda$ is fixed across all iterations whereas $\mu_w$ is doubled with every iteration. 
Our main result is stated below. %We give our main result in the following theorem.

% In each round, the corresponding regularized empirical minimax problem is defined based on a disjoint subset of the entire collection of the datasets.

\begin{algorithm}[h]
\caption{Minimax Phased ERM}
\label{alg:Minimax Phased ERM}
\begin{algorithmic}
\State {\bfseries Input} Sample oracles $\{\cC_1,\dots \cC_p\}$, step size $\eta$%\rnote{no need to define learning rate}
, $D=\max \{L, B\}$, total number of samples drawn from all groups $K$, non-private empirical solver $A_{emp}$.
\State {\bfseries Initialize}  $w_0$ arbitrary parameter vector in $\cW$.
\State {Set $n = K/p$ ~and~ $T = \log_2(n)$.}
\For{$t = 1,\dots, T$}
    \State Let $n_t = n/T,$ %\rnote{if $n_i$ is the same across all rounds why the subscript $i$? why not call it something that does not depend on $i$ and define it once before the for loop?} 
    $\eta_t = \eta 2^{-t}$, $\mu_w^t = 1/(\eta_t n_t)$, $\mu_\lambda^t = 1/(\eta n )$. 
    \State For each $i\in [p]$, sample $S_i^t\sim D_i^{n_t}$ using the sample oracle $\cC_i$. %\rnote{very vague! Directly say ``Sample $S_j^i\sim D_j^{n_i}$ using the sample oracle $O_j$} 
        Denote $\tilde{S}^t = \{S_1^t, \dots S_p^t\}$.
    \State Let $\alpha_t = \frac{L^2}{8n_t^2 \mu_w^t} + \frac{B^2}{8n_t^2\mu_\lambda^t}$ .
    \State Let $\tilde{w}_t = \cA_{emp}(\tilde{S}^t, \mu_w^t, \mu_\lambda^t, w_{t-1}, \alpha_t)$. \\
    \State Set $w_t = \tilde{w_t} + \xi_t$ where $\xi_t\sim \cN(0, \sigma_t I_d)$ and $\sigma_t = 6D\sqrt{2\log_2(n) \log(1/\delta) \eta \eta_t}/\epsilon$.
\EndFor
\State {\bfseries Output:} $w_T$.
\end{algorithmic}
\end{algorithm}
%\rnote{I noticed that you ignore adding full-stops after each line in the code; so I did (except for the FOR statement since it's a loop header). Please, make sure to do so in the other algorithms as well}\xnote{I added full-stops to each line of other two algorithms.}

\begin{thm} \label{thm: expected worst-group excess risk}
     Let $\eta = \frac{M}{D} \min \left\{\frac{\epsilon}{ \sqrt{72d \log (\frac{K}{p}) \log(\frac{1}{\delta})}}, \frac{\sqrt{p}}{\log^{\frac{3}{4}} (K) \sqrt{K}} \right\}$. Algorithm \ref{alg:Minimax Phased ERM} is $(\epsilon, \delta)$-differentially private and %\rnote{Doesn't the privacy guarantee depend on the setting of $\eta$?! The setting of $\eta$ should come first!}\xnote{The privacy does not depend on the setting of $\eta$. I put the setting of $\eta$ in the beginning, but the first line of the theorem becomes very sparse.}
    we have
    \begin{align*}
    &\E\left[\max_{i\in [p]} L_{D_i}(w_T)\right] - \min_{w\in \cW} \max_{i\in [p]} L_{D_i}(w) \\
    &=O\left(MD\left(\log^{\frac{11}{4}} (K)\sqrt{\frac{ p}{K} }
     + \frac{\log^{\frac{5}{2}} (K)p\sqrt{d \log (1/\delta)}}{K\epsilon}\right)\right),
    \end{align*}
   % \rnote{fix the expressions above as we discussed}\xnote{The expression has been fixed}
    %\rnote{What is $q$ above?! Do you mean $p$?}\xnote{It should be $p$. I fixed the typo.}
    where $w_T$ is the output of Algorithm \ref{alg:Minimax Phased ERM} and the expectation is taken over the sampled data points and the algorithm's inner randomness. %Let $K=n\cdot p$ be the total number of samples from all groups. The bound above can also be written as $O\left( \frac{MD p\sqrt{d}}{K\epsilon} + MD\sqrt{\frac{d}{K}} \right)\cdot \text{polylog}(p, 1/\delta, n)$.
%    \xnote{Do we also need to write the bound in terms of $T$, the number of iterations?}\rnote{No. But what is $T$? I don't see any specific setting for $T$ in the algorithm or the theorem statement! Shouldn't $T$ $\log(K/p)$? Everything should work as before if we view $n$ we had before as $K/p$.}\xnote{I rewrote the algorithm and the theorem statement as you suggested in the comments.}
\end{thm}

\paragraph{Remarks:}
\begin{itemize}[leftmargin=*, topsep=0pt]
     \item 
     One can show that the rate in Theorem \ref{thm: expected worst-group excess risk} is nearly optimal in the offline setting considered in Section \ref{sec:offline setting}, where each distribution is observed by a fixed-size dataset of size $K/p$. A lower bound instance of $\Omega\left( \frac{p\sqrt{d}}{K\epsilon} + \sqrt{\frac{p}{K}} \right)$ can be constructed in this case. More details are given in Appendix \ref{sec:lower bound instances}. %\rnote{this is very vague and confusing; the optimality here is w.r.t. the population risk, whereas in section 5 it's w.r.t. the empirical risk. The reference to the offline setting is pretty vague (it can be construed as if we are talking about the ERM setting). You should elaborate by saying ``the offline setting, where each distribution is observed by a fixed-size dataset of size $K/p$''}. \xnote{We write the algorithm in online fashion, which is not suitable for offline setting. Should we say "we can slightly modify Algorithm 2 for the offline setting"?}\rnote{if you elaborate as I suggested above, then there is no need to modify anything. It's clear how the algorithm works; one will simply sample without replacement from the dataset of each group rather than sampling from the sample oracle.}\xnote{I made the modifications as suggested.} 
     Also, note that when $d = O(K/p)$, the rate in Theorem \ref{thm: expected worst-group excess risk} matches the the non-private lower bound $\Omega\left(\sqrt{\frac{p}{K}}\right)$ up to logarithmic factors. 
     %\rnote{I brought this sentence to the end. I think the near optimality remark in the ``offline setting'' is more important.}
    %is nearly optimal up to logarithm factors. Similar to the empirical case, a lower bound instance can be created by constructing one non-trivial distribution with higher population risk than the rest $p-1$ distributions for any $w\in \cW$. We can then reduce the objective into a single-distribution DP-SCO problem, which implies a lower bound of $\Omega\left( \frac{p\sqrt{d}}{K\epsilon} + \sqrt{\frac{d}{K}} \right)$\cite{bassily2014private}. A more detailed instance can be found in Appendix \ref{sec:lower bound instances}.

    \item In Appendix \ref{sec: non-private sc-sc algo}, we give an instantiation of $\cA_{emp}$, which is an iterative algorithm for solving the regularized minimax problem in \eqref{eq:regularized objective} with a convergence rate of $\tilde{O}(1/N)$, where $N$ is the number of iterations. %\rnote{Use a different symbol other than $T$, so that it's not confused with $T$ in our main algorithm. Also, say explicitly that our instantation of $\cA_{emp}$ is iterative}\xnote{I made the changes as suggested.} %\rnote{We should elaborate on the impact of the approximate solution (rather than the exact one) on the analysis. We should make it clear that we can still attain the bound of Theorem~\ref{thm: expected worst-group excess risk} via minor additions to our proof (namely, showing that the approximate saddle point obtained by $\cA_{emp}$ does not have any significant impact on stability or on privacy.}\xnote{All the results in this section are already written in a way to allow approximate solution of $F(w, \lambda)$. We only consider the exact solution in the proof outline of Lemma 6 to simplify the analysis.}
\end{itemize}

The privacy and utility guarantees in Theorem \ref{thm: expected worst-group excess risk} rely on a stability-based result of the saddle point of the regularized objective \eqref{eq:regularized objective}, which we present in the following lemma.

\begin{lem} \label{lem: stability lemma}
    Consider two neighboring dataset collections $\tilde{S} = \{S_1, \dots S_p\}\in \cZ^{n\times p}$ and $\tilde{S}' = \{S'_1, \dots S'_p\}\in \cZ^{n \times p}$ that differs in a single data point. Let $(\tilde{w}, \tilde{\lambda})$ and $(\tilde{w}', \tilde{\lambda}')$ be an $\alpha$-saddle point of $F(w, \lambda)$ in \eqref{eq:regularized objective}   when the dataset collections  are $\tilde{S}$ and $\tilde{S}',$ respectively. By letting $\alpha \leq \frac{L^2}{8n^2 \mu_w} + \frac{B^2}{8n^2\mu_\lambda}$, we have
    \begin{equation*}
        \norm{\tilde{w}-\tilde{w}'}_2 \leq \frac{3}{n}\left(\frac{L}{\mu_w} + \frac{B}{\sqrt{\mu_w \mu_\lambda}}\right)
    \end{equation*}
\end{lem}

The proof of Lemma \ref{lem: stability lemma} follows  %\rnote{the term ``straightforwardly'' make it sound like this basically the same lemma in the other paper (rather than a corollary of it). Even if this is case, do not say ``straightforwardly''. For a skim reader, they may have the impression that this is your central lemma and it follows ``straightforwardly'' from another work} 
from Lemma 2 in \cite{zhang2021generalization} as well as the observation that $F(w,\lambda)$ can be equivalently written in the finite-sum form composed with regularization terms. We give full details in Appendix \ref{sec: proof of Lemma 3}.

Now we are ready to present our main technical lemma, which shows that an approximate saddle point of the regularized objective $F(w, \lambda)$ in \eqref{eq:regularized objective} gives a good solution to the worst-group population risk minimization problem. We will sketch the proof of this lemma and defer the full proof to Appendix \ref{sec: proof of Lemma 4}.

\begin{lem} \label{lem: expected excess loss of hat_w} Let $(\tilde{w}, \tilde{\lambda})$ be an $\alpha$-saddle point of $F(w, \lambda)$ in~\eqref{eq:regularized objective} with $\alpha \leq \frac{L^2}{8n^2 \mu_w} + \frac{B^2}{8n^2\mu_\lambda}$. For any $w\in \cW$,
\begin{align*}
    &\E\left[\max_{i\in [p]} L_{D_i}(\tilde{w})\right] - \max_{i\in [p]} L_{D_i}(w) \\
    &= O \bigg(\mu_w \norm{w- w'}_2^2 + \mu_\lambda \log p + \left(\frac{L^2}{n\mu_w} + \frac{L B}{n\sqrt{\mu_w \mu_\lambda}}\right)\log(n)\log(n p) + \frac{B\sqrt{\log(pn)}}{\sqrt{n}}\bigg),
\end{align*}
where the expectation is taken over the randomness in the datasets $\tilde{S}$. %\rnote{fix the issues here as we discussed}\xnote{I add one sentence to better explain the lemma.}
\end{lem}
\begin{proof} (sketch)
For simplicity, we let $\alpha = 0$, which means $(\tilde{w}, \tilde{\lambda})$ is the exact saddle point of $F(w, \lambda)$. In particular, by the definition of saddle point, one can show that for any $w\in \cW$
\begin{equation*}
    \max_{i \in [p]} L_{S_i}(\tilde{w}) - \hat{\phi}(w, \tilde{\lambda}) \leq \frac{\mu_w}{2} \norm{w- w'}_2^2 + \mu_\lambda \log p
\end{equation*}
where $\hat{\phi}(w,\lambda) = \sum_{i=1}^p \lambda_i L_{S_i}(w)$.

By Lemma \ref{lem: stability lemma} and Lipschitzness, for any given $i\in [p]$, computing $\tilde{w}$ has a uniform stability of $\gamma = \frac{3}{n}\left(\frac{L^2}{\mu_w} + \frac{LB}{\sqrt{\mu_w \mu_\lambda}}\right)$ with respect to the change of datapoint in $S_i$. By Theorem 1.1 in \cite{feldman2019high}, we have with high probability over the sampling of $\tilde{S}$, 
\begin{equation*}
    |L_{S_i}(\Tilde{w}) - L_{D_i}(\tilde{w})| = \tilde{O}\left(\gamma + \frac{B}{\sqrt{n}}\right)
\end{equation*}

Then we can use union bound across all $i\in [p]$ and obtain with high probability
\begin{equation*}
    |L_{S_i}(\Tilde{w}) - L_{D_i}(\tilde{w})| = \tilde{O}\left(\gamma + \frac{B}{\sqrt{n}}\right) \quad \forall i\in [p]
\end{equation*}

In particular, we have with high probability
\begin{equation*}
    |\max_{i \in [p]} L_{D_i}(\tilde{w}) - \max_{i\in [p]} L_{S_i}(\tilde{w})| = \tilde{O}\left(\gamma + \frac{B}{\sqrt{n}}\right)
\end{equation*}

Meanwhile, we can show that with high probability, $|\phi(w, \tilde{\lambda}) - \hat{\phi}(w, \tilde{\lambda})|= \tilde{O}\left(\frac{B}{\sqrt{n}}\right)$ for any given $w\in \cW$. Finally, we put everything together and have with high probability
\begin{align*}
    &\max_{i\in [p]} L_{D_i}(\tilde{w}) - \max_{i \in [p]} L_{D_i}(w) \\
    &\leq \max_{i\in [p]} L_{D_i}(\tilde{w}) - \phi(w, \tilde{\lambda}) \\
    &\leq \max_{i\in [p]} L_{S_i}(\tilde{w}) - \hat{\phi}(w, \tilde{\lambda}) + \tilde{O} \left(\gamma + \frac{B}{\sqrt{n}}\right)  \\
    &= \tilde{O} \left(\mu_w \norm{w- w'}_2^2 + \mu_\lambda + \gamma + \frac{B}{\sqrt{n}}\right)
\end{align*}

Replacing $\gamma$ with $\frac{3}{n}\left(\frac{L^2}{\mu_w} + \frac{LB}{\sqrt{\mu_w \mu_\lambda}}\right)$ and we obtain the desired result.
\end{proof}

\paragraph{Remark:} As a side product of Lemma \ref{lem: expected excess loss of hat_w}, we can sample a new dataset collection $\tilde{S} = \{S_1,\dots S_p\}$ where $S_i\sim D_i^{K/p}$. By letting $\mu_w = \frac{D}{M}\sqrt{\frac{p}{K}}\sqrt{\log (K/p) \log (K)}$, $\mu_\lambda = D\sqrt{\frac{p\log (K/p) \log(K)}{K\log p}}$ and $w'$ be any arbitrary parameter in $\cW$, solving the regularized objective \eqref{eq:regularized objective} with $\tilde{S}$ gives an excess worst-group population risk of $\tilde{O}\left(\sqrt{\frac{p}{K}}\right)$, which matches the non-private lower bound in \cite{soma2022optimal} up to logarithmic factors. Prior methods (nearly) matching the non-private lower bound depend on making one pass over the sampled data points to directly optimize the population objective, our regularization-based method however does not require such one pass structure.

Now, we are ready to give a proof sketch of Theorem \ref{thm: expected worst-group excess risk}. The full proof is deferred to Appendix \ref{sec:proof of Theorem 2}.
\begin{proof}
    (Proof Sketch of Theorem \ref{thm: expected worst-group excess risk})
    
\paragraph{Privacy:} %\sout{We denote $n=2^T$, then we have $n_i = n/\log_2 n$, $\mu_w^i = 1/(\eta n)$ and $K = n\cdot p$.}\rnote{$T$ is not defined anywhere! $T$ should be $\log(K/p)$ and $n$ should be defined in the proof as $K/p$.}\xnote{I will define $n$ and $T$ inside the algorithm.} 
At iteration $t$, by the stability argument in Lemma~\ref{lem: stability lemma}, we have
\begin{align*}
    &\norm{\tilde{w_t} - \tilde{w_t}'}_2 \leq \frac{3L}{n_t \mu_w^t} + \frac{3B}{n_t \sqrt{\mu_w^t \mu_\lambda^t}} \\
    &= 3L \eta_t + \frac{3B \sqrt{n_t n \eta_t \eta}}{n_t} \leq 6D \sqrt{(\log_2 n) \eta_t \eta}
\end{align*}
where $\tilde{w_t}$ and $\tilde{w_t}'$ are outputs from neighboring dataset collections $\tilde{S}^t$ and $\tilde{S}^{t'}$ that differ in one datapoint. The privacy follows from the privacy guarantee of Gaussian mechanism.

\paragraph{Utility:} Recall that $R(w) = \max_{i\in [p]} L_{D_i}(w)$. By Lemma~\ref{lem: expected excess loss of hat_w}, we have
\begin{align*}
    \E[R (\tilde{w_t}) - R (\tilde{w}_{t-1})] =\tilde{O}\left( \frac{\E[\norm{\xi_{t-1}}_2^2]}{n_t \eta_t}  + \frac{1}{\eta n} + D^2\sqrt{\eta \eta_t} + \frac{B}{\sqrt{n_t}} \right)
\end{align*}
.

Also, we have 
\begin{equation*}
    \E \norm{\xi_t}_2^2 = d\sigma_t^2 = 72d D^2 \log n \log (1/\delta) 2^{-t} \eta^2 /\epsilon^2
\end{equation*}
Therefore, as long as
\begin{equation*}
    \eta \leq \frac{M\epsilon}{D \sqrt{72d \log n \log(1/\delta)}}
\end{equation*}
We will have 
\begin{equation*}
    \E \norm{\xi_t}_2^2 \leq 2^{-t} M^2 \text{ and } \E \norm{\xi_t}_2 \leq \sqrt{2}^{-t} M
\end{equation*}

Let $\tilde{w}_0 = w^*$ and $\xi_0 = w_0 - w^*$. Then $\norm{\xi_0}_2\leq M$. Hence,
\begin{align*}
    &\E [R(w_T)] - R(w^*) \\
    &= \sum_{t=1}^T  \E [R(\tilde{w}_t) - R(\tilde{w}_{t-1})] + \E [R(w_T) - R(\hat{w}_{T})] \\
    &\leq \sum_{t=1}^T \tilde{O}\left( \frac{\E[\norm{\xi_{t-1}}_2^2]}{n \eta_t}  + \frac{1}{\eta n} + D^2\sqrt{\eta \eta_t} + \frac{B}{\sqrt{n}} \right)+ L \E[\norm{\xi_T}_2] 
\end{align*}
%\rnote{the equation above is spilling over the text. Please, fix.}
The inequality holds since $R(\cdot)$ is $L$-Lipschitz and $n_t = n/\log_2(n)$.

We have 
\begin{equation*}
    \frac{\E[\norm{\xi_{t-1}}_2^2]}{n_t \eta_t} \leq \frac{2^{-(t-1)}M^2\log_2 n}{n 2^{-t} \eta} = \frac{2(\log_2 n) M^2}{\eta n}
\end{equation*}
and $\E[\norm{\xi_T}_2] \leq  M \sqrt{2}^{-\log_2 n} = \frac{M}{\sqrt{n}}$. 

Therefore,
\begin{align*}
    &\E [R(w_T)] - R(w^*) \\
    &\leq \sum_{t=1}^T \tilde{O}\left( \frac{M^2}{\eta n} + D^2\eta  + \frac{B}{\sqrt{n}} \right) +\frac{ML}{\sqrt{n}} \\ 
    &=\tilde{O}\left(\frac{MD}{\sqrt{n}} +  \frac{MD \sqrt{d}}{n\epsilon}\right)
\end{align*}
Replace $n$ with $K/p$ and we get the desired result.
\end{proof}
\section{Worst-group Risk Minimization using DP OCO Algorithm} \label{sec: DP OCO} 
In this section, we give a framework to directly minimize the expected worst-group loss in \eqref{eq:agnostic loss 2} using any DP online convex optimization (OCO) algorithm as a subroutine. 

Our framework leverages the idea from \cite{haghtalab2022demand,soma2022optimal, zhang2023stochastic} of casting the objective into a two-player zero-sum game. One can show that the excess expected worst-group risk is bounded by the sum of the expected regret bounds of the min-player (the $w$-player) and max-player (the $\lambda$-player) using stochastic gradient estimation. Therefore, given any DP-OCO algorithm $\cQ_-$, our frameworks proceed by letting the $w$-player run $\cQ_-$ with an estimate of the loss function $\phi(w, \lambda_t)$. Meanwhile, we instantiate the $\lambda$-player as a standard EXP3 multi-arm bandit algorithm and feed it with a privatized estimator of $\nabla_\lambda \phi(w_t, \lambda_t)$. 

An online convex optimization algorithm interacts with an adversary for $T$ rounds. In each round, the algorithm picks a vector $w_t$ and then the adversary chooses a convex loss function $\ell_t$. The performance of an OCO algorithm $\cQ$ is therefore measured by the regret defined as
\begin{equation*}
    r_\cQ(T)= \frac{1}{T} \,\E\left[\sum_{t=1}^T \ell_t(w_t) - \min_{w^*\in \mathcal{W}} \sum_{t=1}^T \ell_t(w^*)\right]
\end{equation*} 

\begin{defn}
An online algorithm $\cM$ is $(\epsilon, \delta)$-DP if for any two sequence of loss functions $S = (\ell_1, \dots \ell_T)$ and $S' = (\ell'_1, \dots \ell'_T)$ that differs in one loss function, we have 
\begin{equation*}
    P(\cM(S)\in O) \leq e^\epsilon P(\cM(S') \in O) + \delta
\end{equation*}
%\rnote{why do we need another definition for DP?! Doesn't the one in the prelims cover the online case? Are you trying to give a stronger result where the loss is not fixed across iterations? If so, you need to highlight this strength explicitly (in contrast to the previous section) and give an explicit justification for the new version of the DP definition}\xnote{This definition is for the general online learning scenario where the inputs are loss functions instead of data points. In our case, the loss function $\ell_t(\cdot)=\ell(\cdot, x_t)$, but it is not always this way.}
\end{defn}

Suppose we have a DP OCO algorithm $\cQ_-$ with regret $r_{\cQ_-}(T)$ that observes $l_t = \ell(\cdot, x_t)$ for some data point $x_t$ in each iteration, the formal description of our framework is given in Algorithm \ref{alg:on demand sampling}. We give the privacy and utility guarantee of Algorithm \ref{alg:on demand sampling} in the following theorems.  
\begin{algorithm} [h]
\setstretch{1.15}
\caption{Worst-group risk minimization using DP-OCO algorithm}
\label{alg:on demand sampling}
\begin{algorithmic}
\State {\bfseries Input} Sample oracles $\{\cC_1, \dots \cC_p\}$, DP OCO algorithm $\cQ_-$, learning rate $\eta$, privacy parameters $\epsilon, \delta$.
\State {\bfseries Initialize}  $w_1 \in \mathcal{W}$ and $\lambda_1=\left(1/p,\dots 1/p\right)\in [0,1]^p$.
\State Set $U= B + \frac{2B}{\epsilon}\log(T)$.
\For{$t = 1,\dots, T-1$}
    \State Sample $i_t \sim \lambda_t$.
    \State Sample $x_{t}^-, x_{t}^+ \sim D_{i_t}$ using the sample oracle $\cC_{i_t}$. 
    \State Update $w_{t+1} = \cQ_-(w_{1:t},  x_{t}^-)$. 
    \State Compute $\Tilde{\ell}_t = U - \ell(w_t, x_{t}^+) + y_t,$~ $y_t\sim \text{Lap}(B/\epsilon)$.
    \State Update $\lambda_{t+1, i_t} = \lambda_{t, i_t}  \text{exp}\left(\frac{-\eta\Tilde{\ell}_{t}}{\lambda_{t,i_t}}\right)$. 
    \State  Normalize $\lambda_{t+1} = \frac{\lambda_{t+1}}{\norm{\lambda_{t+1}}_1}$.
\EndFor
\State {\bfseries Output} {$\Bar{w}_T = \frac{1}{T} \sum_{t=1}^T w_t$, $\Bar{\lambda}_T = \frac{1}{T} \sum_{t=1}^T \lambda_t$.}

\end{algorithmic}
\end{algorithm}

\begin{thm}(Privacy Guarantee)
Algorithm \ref{alg:on demand sampling} is $(\epsilon, \delta)$-differentially private.
\end{thm}
%\rnote{What do you mean assume it is DP? Do you mean ``show'' not ``assume''? or are you proving the privacy via induction? If you are using induction, then your proof should be more formal and elaborate. You should first start by saying that you will prove this by induction over the iterations. You should then mention the base case is trivially satisfied, then proceed to show the induction step. ALSO, it is not formal to say that a parameter or an iterate is DP; DP is a property of a computation not the result of a computation. You may say that you show that $w_t$ and $\lambda_t$ are generated in $(\epsilon, \delta)$-DP manner.}\xnote{the proof is revised}.

\begin{proof}
The sampled dataset can be divided into two disjoint sets. $S = \{S_-, S_+\}$ where $S_- = \{x_{1}^-,\dots x_{T}^-\}$ and $S_+ = \{x_{1}^+,\dots x_{T}^+\}$. %\rnote{say: Now, we proceed inductively}\xnote{Fixed.} 
Now, we proceed inductively. When $t=1$,  it is easy to see that the generation of $(w_1, \lambda_1)$ is $(\epsilon, \delta)$-DP. 

Now we assume that the generation of $\{(w_i, \lambda_i)\}_{i=1}^t$ is $(\epsilon, \delta)$-DP for iteration $t$. At iteration $t+1$, we observe that the computation of $w_{t+1}$ only depends on $S_-$ and $\{w_i\}_{i=1}^t$, hence $w_{t+1}$ can be generated under $(\epsilon, \delta)$-DP constraint due to the privacy guarantee of $\cQ_-$.  

Furthermore, since $\ell$ is uniformly bounded by $B$, then clearly the sensitivity of $\ell(w_t, \cdot)$ w.r.t. replacing one data point in the input sequence is bounded by $B$. Hence, by the properties of the Laplace mechanism, $\Tilde{\ell}_{t}$ is also generated in $(\epsilon, 0)$-DP manner. Given that $\lambda_{t+1}$ depends on only $\Tilde{\ell}_t$ and $\lambda_t$ and since DP is closed under postprocessing, computing $\lambda_{t+1}$ is $(\epsilon, \delta)$-DP. Since in iteration $t$, two different data points (namely, $x_t^-$ and $x_t^+$) are used to generate $w_{t+1}$ and $\lambda_{t+1}$, then by parallel composition (and given the induction hypothesis), generating $(w_{t+1}, \lambda_{t+1})$ is $(\epsilon, \delta)$-DP.
%Therefore,generating  $(w_{t+1}, \Tilde{\ell}_{t})$ is under $(\epsilon, \delta)$-DP constraint \rnote{This is ok, but reviewers may complain about this being a bit informal and short. If there is time, should add a more detailed proof in the appendix.}
\end{proof}

\begin{thm} (Utility Guarantee) \label{thm: utility guarantee online learning}
    Suppose the regret of $\cQ_-$ is denoted as $r_{\cQ_-}(\cdot)$. By setting $\eta = \sqrt{\frac{\ln(p)}{pTU^2}}$,  we have 

    \begin{align*}
        \E\left[\max_{i\in [p]} L_{D_i}(\bar{w}_T)\right] - \min_{w\in \cW} \max_{i\in [p]} L_{D_i}(w)= r_{\cQ_-}(K/2)  +O\left(\frac{B\log(K)}{\epsilon}\sqrt{\frac{p\log(p)}{K}}\right)
    \end{align*}
    where $\bar{w}_T$ is the output of Algorithm \ref{alg:on demand sampling} and the expectation is over the sampling of data points and the algorithm's randomness.
\end{thm}
The proof of Theorem \ref{thm: utility guarantee online learning} can be found in Appendix \ref{sec:proof of DP OCO}. In particular, by instantiating the DP OCO algorithm $\cQ_-$ with the
DP-FTRL algorithm \cite{kairouz2021practical}, we have the following corollary.
\begin{cor} \label{cor: DP-FTRL}
Let $\cQ_-$ in Algorithm \ref{alg:on demand sampling} be DP-FTRL algorithm from \cite{kairouz2021practical}. By plugging the regret bound of DP-FTRL into Theorem \ref{thm: utility guarantee online learning}, we have 
\begin{align*}
&\E\left[\max_{i\in [p]} L_{D_i}(\bar{w}_T)\right] - \min_{w\in \cW} \max_{i\in [p]} L_{D_i}(w) \\
&= O\left(ML\sqrt{\frac{d^{\frac{1}{2}}\log^2(1/\delta)\log(K)}{\epsilon K}} +\frac{B\log(K)}{\epsilon}\sqrt{\frac{p\log(p)}{K}}\right),
\end{align*}
%where $\bar{w}_T$ is the output of Algorithm \ref{alg:on demand sampling}. 
where the expectation is over the sampling of data points and the algorithm's randomness.
\end{cor}

\paragraph{Remark:} Corollary \ref{cor: DP-FTRL} demonstrates an excess expected worst-group risk of $\Tilde{O}\left( \sqrt{\frac{d^{1/2}}{\epsilon K}} +\sqrt{\frac{p}{K\epsilon^2}} \right)$. For the common case of $\epsilon=\Theta(1)$, the rate in Corollary \ref{cor: DP-FTRL} matches the lower bound of $\Omega\left( \sqrt{\frac{p}{K}} \right)$ shown in \cite{soma2022optimal} up to logarithmic factors in the regime of $d = \Tilde{O}\left( p^2 \right)$.
    
%The update rule for the group weight vector $\lambda$ can be replaced by other adversarial multi-armed bandit algorithms like Tsallis-INF \cite{zimmert2021tsallis}  or EXP3-IX \cite{neu2015explore} with same convergence rate (up to logarithmic factors) based on the regret bound of each algorithm.

\section{Private Worst-group Empirical Risk Minimization} \label{sec:offline setting}
In this section, we consider the \emph{offline setting} where each distribution $D_i$ is observed by a fixed-size dataset $S_i$ with i.i.d samples drawn from $D_i$. In the offline setting, the learner will take the dataset collection $\tilde{S} = \{S_1, \dots S_p\}$ as input instead of querying the sample oracles directly. This setting captures scenarios where each group is represented by a dataset whose size is fixed beforehand. %The offline setting can be seen as a special case of the sample oracle model in previous sections and is motivated by the scenarios when the sampling process is either expensive or time-consuming such that assuming a sample oracle that returns new samples during the learning process is not realistic. Therefore, a practical way to perform learning under these scenarios is to collect a dataset beforehand for each distribution and run a worst-group risk minimization algorithm with these pre-collected datasets as input.
%\rnote{ the motivation is not convincing; the sample oracle is more general: we can view querying a data point from the sample oracle simply as sampling from a fixed dataset; moreover, our Phased-ERM algorithm is actually an offline one, it does not require adaptive sampling, so our motivation is not really well justified. I think we should remove this. We should just say that this setting captures scenarios where each group is represented by a dataset whose size is fixed beforehand.}\xnote{I removed the motivation and add the sentence as suggested.}

In particular, we denote the dataset collection as  $\tilde{S}=\{S_1, \dots S_p\}$. Without loss of generality, we assume all datasets $S_i$ have same size, namely $|S_i| = n$ for all $i\in [p]$. Note that $n$ can also be written as $n = K/p$ in this case. We first define the worst-group empirical risk minimization problem as follows.

\paragraph{Worst-group empirical risk minimization}:
Given a dataset collection $\tilde{S} = \{S_1, \dots S_p\}\in \cZ^{n\times p}$, we denote $ \hat{\phi}(w, \lambda) = \sum_{i=1}^p \lambda_i L_{S_i}(w)$. The worst-group empirical loss can be expressed as
$\max_{i\in [p]} L_{S_i}(w)= \max_{\lambda\in \Delta_p}  \hat{\phi}(w, \lambda)$. The empirical objective is hence written as
\begin{equation}
    \min_{w\in \cW} \max_{\lambda\in \Delta_p} \hat{\phi}(w, \lambda) \label{eq:empirical agnostic loss}
\end{equation}
We define the \emph{excess worst-group empirical risk} of $w\in \cW$ as
\begin{equation*}
    \hat{\cE}(w; \{S_i\}_{i=1}^p)\triangleq \max_{i\in [p]} L_{S_i}(w) - \min_{\tilde{w}\in \cW} \max_{i\in [p]} L_{S_i}(\tilde{w}).
\end{equation*}

Next, we will present a method described in Algorithm \ref{alg:Noisy-SGD-MW} to solve the empirical objective \eqref{eq:empirical agnostic loss} under the differential privacy constraint. Algorithm \ref{alg:Noisy-SGD-MW} is based on a private version of the multiplicative reweighting scheme \cite{abernethy2022active, diana2021minimax} and attains nearly optimal excess empirical worst-group risk.

\subsection{Private Multiplicative Group Reweighting}

Here, we provide the details of our algorithm described formally in Algorithm \ref{alg:Noisy-SGD-MW}. In this algorithm, we maintain a sampling weight vector $\lambda\in [0,1]^p$ for all datasets $\{S_i\}_{i=1}^p$. In each iteration, we first sample a dataset according to $\lambda$.  Next, we sample a mini-batch from this dataset and update the model parameters using Noisy-SGD. We then update the group weight vector $\lambda$ using the multiplicative weights approach based on privatized loss values computed over the datasets at the model parameters. We present the privacy and utility guarantee of Algorithm \ref{alg:Noisy-SGD-MW} in the following theorems.

\begin{algorithm}[tb]
\caption{Noisy SGD with Multiplicative Group Reweighting (Noisy-SGD-MGR)}
\label{alg:Noisy-SGD-MW}
\begin{algorithmic}
\State {\bfseries Input:} Collection of datasets $\tilde{S} = \{S_1, \dots S_p\}\in \cZ^{n\times p}$, mini-batch size $m$, $\#$ iterations $T$, learning rates $\eta_w$ and $\eta_\lambda$, privacy parameters $\epsilon, \delta$, and noise scales $\sigma^2, \tau$.
\State {\bfseries Initialize} $w_1\in \cW$ and $\lambda_1 = \left(\frac{1}{p}, \dots \frac{1}{p}\right)\in [0, 1]^p$.\\
\For{$t = 1,\dots, T-1$}
\State Sample $i_t\sim \lambda_t$. 
\State Sample $B_t = \{z_1,\dots z_m\}$ from $S_{i_t}$ uniformly with replacement. 
\State Update the model as follows 
\begin{equation*}
    w_{t+1}
    = \text{Proj}_{\cW}\left(w_t - \eta_w \cdot \left(\frac{1}{m}\sum_{z\in B_t} \nabla\ell(w_t, z) +G_t \right) \right),
\end{equation*}
where $G_t \sim \mathcal{N}(0, \sigma^2 I_d)$. 
\State Compute $L_t = [-L_{S_i}(w_t)+y_{i, t}]_{i=1}^q,$~ $y_{i,t}\stackrel{\text{iid}}{\sim}\text{Lap}(\tau)$.
\State Update the weights $\tilde{\lambda}_{t+1}^i = \lambda_t^i \exp(-\eta_\lambda L_t^i),$~ $\forall~i \in [p]$. \\
\State Normalize $\lambda_{t+1} = \frac{\tilde{\lambda}_{t+1}}{\norm{\tilde{\lambda}_{t+1}}_1}$.
\EndFor

\State {\bfseries Output:} $\Bar{w}_T = \frac{1}{T} \sum_{t=1}^T w_t$, $\Bar{\lambda}_T = \frac{1}{T} \sum_{t=1}^T \lambda_t$.
\end{algorithmic}
\end{algorithm}

\begin{thm} (Privacy guarantee) \label{thm: privacy gurantee of Noisy-SGD-MGR}
In Algorithm \ref{alg:Noisy-SGD-MW}, let $\tau = \frac{c Bp}{K\epsilon}\sqrt{T\log(1/\delta)}$ and $\sigma^2 = \frac{c Tp^2L^2\log(T/\delta)\log(1/\delta)}{K^2 \epsilon^2}$ for some universal constant $c$. %\xnote{Can we do moment accountant for both Laplace and Gaussian mechanisms?}
%\rnote{there should not be a comma here! The sentence has ended. Replace with full stop, then start a new sentence ``Then, Algorithm~\ref{alg:Noisy-SGD-MW} is $(\epsilon, \delta)$-DP''} 
Then, Algorithm \ref{alg:Noisy-SGD-MW} is $(\epsilon, \delta)$-DP. %\rnote{Define $\tau$ and $\sigma$ inside the algorithm not here. It is very hard to find these parameters since they are not defined as input parameters! Aletrnatively (though less ideal), define $\sigma^2$ and $\tau$ as input parameters in the input line of the alg., and then start the theorem statement by saying: In Alg. 1, let $\sigma^2= ...$ and $\tau = ...$, then Alg. 1 is $(\epsilon, \delta)$-DP. Btw, you should also put $\epsilon, \delta$ in the input line as privacy parameters}\xnote{I added $\epsilon, \delta$ and $\sigma^2, \tau$ to the input line and gives the values of $\sigma^2, \tau$ in the theorem statement.}
\end{thm}

\begin{thm} (Convergence guarantee) \label{thm:convergence rate Noisy-SGD-MGR}
There exist settings of $T = O\left(\frac{(ML + B\sqrt{\log(p)})K^2\epsilon^2}{GBdp^2\log(1/\delta)}\right)$, $\eta_w = O\left(\frac{M^2}{T(G^2 + d\sigma^2)}\right)$ and $\eta_\lambda = O\left(\sqrt{\frac{\log (p)}{U^2 T}}\right),$ where $U = O\left(B + \frac{Bp}{K\epsilon}\sqrt{T\log(1/\delta)}\log(KT)\right)$ 
%\rnote{It's not ideal to have the parameter settings defined in terms of the $O(\cdot)$ notation. If you cannot find constant factors, then you should say: ``There exist settings of $T=O(...), \eta_w=O(...)$, etc. such that we have ...'' rather than saying ``By letting $T=O(..)$ ...''}
such that  %the output of Algorithm~\ref{alg:Noisy-SGD-MW} satisfies %\xnote{Is it okay to have $T$ in the bound?}\rnote{it's better to write it in terms of the above setting you state for $T$, i.e., in terms of $n, \log(p), \epsilon, \log(1/\delta) d$. So, inside the log in the bound below we should have $\frac{np\epsilon}{d\log(1/\delta)}$}
\begin{align*}
    &\E\left[\max_{i\in [p]} L_{S_i}(\bar{w}_T)\right] - \min_{w\in \cW} \max_{i\in [p]} L_{S_i}(w) \\
     &= O\bigg(\frac{MLp\sqrt{d\log(1/\delta)\log(K/\delta)}}{K\epsilon} + \frac{Bp\sqrt{\log(p)\log(1/\delta)}\log(\frac{K\epsilon}{d\log (1/\delta)})}{K\epsilon} \bigg), 
\end{align*}
where  the expectation is over the algorithm's  randomness. %\rnote{made few edits in the theorem statement}
%\rnote{You should state the expectation is w.r.t. what; namely, the algorithm's internal randomness.}\xnote{added the phrase}
%\rnote{What is $\kappa$ in the polylog term?! We should also explicitly write the polylog term in the main bound above, but keep it as polylog in the sentence below.}\xnote{$\kappa$ is the probability I used to bound the norm of $L_t$, I replaced it with $1/n$ in the analysis. I write the bound in a more explicit manner. In (6), the second term is insignificant compared with the first term, and will be absorbed when we write the rate in polylog.}
%\rnote{the following sentence should not be a part of the theorem. you should introduce $K= np$ from the preliminaries (or early in that section), then just state the equivalent expression for the bound.}
%\rnote{the last expression with the polylog is actually a bit more loose than the actual bound. I don't think you should write the ``simplified/loose'' bound with the polylog expression; just keep the original bound. If you want to write it in terms of $K, p$ in a simplified way, you can write $\widetilde{O}\left(\frac{GM p\sqrt{d}}{K \epsilon}\right)$} \rnote{Btw, there is no need to number the bound. You can just refer to it as the bound of Theorem 2. I removed it and made the necessary changes.}
\end{thm}

\noindent \textbf{Remarks:}

\begin{itemize}[leftmargin=*, topsep=0pt]
    \item The convergence rate in  Theorem~\ref{thm:convergence rate Noisy-SGD-MGR} is nearly optimal up to logarithmic factors. A lower bound instance can be created to reduce the original problem to a DP-ERM problem with single dataset,
    %We can create a lower bound instance by one dataset with empirical risk higher than the rest $p-1$ datasets for any $w\in \cW$. Therefore, we can reduce the objective into a a setting of DP-ERM with a single dataset of size $n$ %\rnote{ there is no distribution here. This is empirical risk minimization. It may better to say ``reduce it to a setting of DP-ERM with a single dataset of size $n$''}\xnote{I changed the phrase} 
     which leads to a lower bound of $\Omega\left( \frac{p\sqrt{d}}{K\epsilon} \right)$\cite{bassily2014private}. A more detailed argument for the lower bound instance can be found in Appendix~\ref{sec:lower bound instances}.
    %\rnote{you should cite [BST14] for the lower bound on DP-ERM}\xnote{I added the citation}.
    \item  \cite{abernethy2022active} proposes a gradient method based on the active group selection scheme. One can also design a private algorithm based on this scheme and achieve a suboptimal %\rnote{this is not slightly worse; this is worse. For example, think of a regime when $d=O(1)$.}\xnote{removed 'slightly'} 
     rate of $\Tilde{O}\left(\sqrt{\frac{MLBp}{K\epsilon}}+\frac{MLp\sqrt{d}}{K\epsilon}\right)$. We defer the details %of this method 
     to Appendix \ref{sec:DP AGS}. %\xnote{Do we need to move this result into introduction?} \rnote{I think it would be ok to leave it here.}
\end{itemize}

\section{Conclusion}
We presented differentially private algorithms for solving the worst-group risk minimization problem. In particular, we gave two upper bounds on the excess worst-group population risk, one of them is tight in the offline setting and we established the optimal rate for the excess worst-group empirical risk. Establishing the optimal rate for the worst-group population risk in general is left as an interesting open problem for future work. 

\section*{Acknowledgements}
This work is supported by NSF Award 2112471 and NSF CAREER Award 2144532. 
% \section*{Impact Statement}
% This work aims at advancing the theoretical understanding of privacy-preserving machine learning, which is an area of clear and well-established societal impacts. There is nothing on the societal impact of this work which we feel must be specifically highlighted here.

\printbibliography
\newpage
\appendix

\section{Missing Proofs in Section \ref{sec:minimax phased ERM}} \label{sec: proof of minimax phased ERM}
\subsection{Non-private algorithm for the regularized objective } \label{sec: non-private sc-sc algo}
Here, we provide an non-private algorithm in solving the regularized minimax problem in \eqref{eq:regularized objective} written as
\begin{equation*} 
    \min_{w\in \cW} \max_{\lambda\in \Delta_p} F(w, \lambda) \coloneqq \sum_{i=1}^p \lambda_i L_{S_i}(w) + \frac{\mu_w}{2} \norm{w- w'}^2 - \mu_\lambda \sum_{j=1}^p \lambda_j \log \lambda_j
\end{equation*}
with an convergence rate of $\tilde{O}\left(1/N\right)$ where $N$ is the number of iterations. 

\begin{algorithm}
\caption{Minimax Optimization for SC-SC Objective}
\label{alg:non-private SC-SC}
\begin{algorithmic}
\State {\bfseries Input: }Collection of datasets $S = \{S_1, \dots S_p\}$, regularization parameters $\mu_w$, $\mu_\lambda> 0$.
\State \textbf{Init}: $w_1\in \cW$.
\For{$t = 1,\dots, N-1$}
    \State Compute $\tilde{\lambda}_t = \left[ \exp(L_{S_1}(w_t)/\mu_\lambda), \dots \exp(L_{S_p}(w_t)/\mu_\lambda)    \right]$. 
    \State Let $\lambda_t = \frac{\tilde{\lambda_t}}{\norm{\lambda_t}_1}$. 
    \State Compute $\nabla_t = \sum_{i=1}^p \lambda_t^i \nabla L_{S_i}(w_t) + \mu_w (w_t-w')$ .
    \State Update $w_{t+1} = \text{Proj}_\cW (w_t - \eta_t \nabla_t)$.
\EndFor
\State {\bfseries Output:} $\Bar{w}_N = \frac{1}{N} \sum_{t=1}^N w_t$ and $\Bar{\lambda}_N = \frac{1}{N} \sum_{t=1}^N \lambda_t$ .

\end{algorithmic}
\end{algorithm}

We now provide the convergence guarantee on Algorithm \ref{alg:non-private SC-SC}.
\begin{thm}
    Let $\eta_t = \frac{1}{\mu_w t}$, we have 
    \begin{equation*}
        \max_{\lambda \in \Delta_p} F(\Bar{w}_N, \lambda) - \min_{w\in \cW} F(w, \bar{\lambda}_N) = O\left(\frac{\ln N}{N} \left(\frac{L^2}{\mu_w} + M^2 \mu_w\right)  \right)
    \end{equation*}
\end{thm}

\begin{proof}
    An important observation here is that in $\lambda_t$ is the best response of function $F(w_t, \cdot)$, that is,
    \begin{equation*}
        \lambda_t = \argmin_{\lambda\in \Delta_p} F(w_t, \lambda)
    \end{equation*}

    Then by Corollary 11.16 in \cite{orabona2019modern}, we have 
    \begin{equation} \label{eq: sc-sc convegence proof - 1}
        \max_{\lambda \in \Delta_p} F(\Bar{w}_N, \lambda) - \min_{w\in \cW} F(w, \bar{\lambda}_N) \leq \frac{\text{Regret}_N^w}{N}
    \end{equation}
    where $\text{Regret}_N^w = \sum_{t=1}^N l_t(w_t) - \min_{w\in \cW}  \sum_{t=1}^N l_t(w)$ and $l_t(w) = F(w, \lambda_t) $.
    Since $w$ is updated using online gradient descent and $l_t(w)$ is strongly convex, by Corollary 4.9 in \cite{orabona2019modern} and $\norm{\nabla_t}_2 \leq G + \mu_w B$, we have
    \begin{equation} \label{eq: sc-sc convegence proof - 2}
        \text{Regret}_N^w = O\left(\frac{(L + \mu_w M)^2}{\mu_w}(1+\ln N)\right) = O\left(\ln N \left(\frac{L^2}{\mu_w} + M^2 \mu_w\right) \right)
    \end{equation}
    Combining equations \eqref{eq: sc-sc convegence proof - 1} and \eqref{eq: sc-sc convegence proof - 2}, we get the desired result.
\end{proof}

\subsection{Auxiliary Lemma}
\begin{lem} \label{lem: appximation distance}
    Let $(\hat{w},\hat{\lambda})$ be the exact saddle point of \eqref{eq:regularized objective} and $(\tilde{w}, \tilde{\lambda})$ be an $\alpha$-saddle point of \eqref{eq:regularized objective}. Then we have
    \begin{equation*}
        \norm{\hat{w}-\tilde{w}}_2 \leq \sqrt{\frac{2\alpha}{\mu_w}}
    \end{equation*}
\end{lem}

\begin{proof}
    Since $(\tilde{w}, \tilde{\lambda})$ is an $\alpha$-saddle point of $F(w, \lambda)$, then
    \begin{align*}
        F(\tilde{w}, \hat{\lambda}) - F(\hat{w}, \tilde{\lambda}) &\leq \alpha \\
        \implies F(\tilde{w}, \hat{\lambda}) - F(\hat{w}, \hat{\lambda})+ F(\hat{w}, \hat{\lambda})  - F(\hat{w}, \tilde{\lambda}) &\leq \alpha
    \end{align*}

    Since $(\hat{w},\hat{\lambda})$ is the saddle point, then 
    \begin{equation*}
         F(\hat{w}, \hat{\lambda})  - F(\hat{w}, \tilde{\lambda}) \geq 0
    \end{equation*}
    Therefore,
    \begin{equation*}
        F(\tilde{w}, \hat{\lambda}) - F(\hat{w}, \hat{\lambda}) \leq \alpha
    \end{equation*}

    Also we have $\hat{w} = \argmin_{w\in \cW} F(w, \hat{\lambda})$ and $F(\cdot, \hat{\lambda})$ is a $\mu_w$-strongly convex function, then
    \begin{align*}
        &\mu_w/2 \norm{\hat{w}-\tilde{w}}_2^2 \leq F(\tilde{w}, \hat{\lambda}) - F(\hat{w}, \hat{\lambda}) \leq \alpha \\
        &\implies \norm{\hat{w}-\tilde{w}}_2 \leq \sqrt{\frac{2\alpha}{\mu_w}}
    \end{align*}
\end{proof}

\subsection{Proof of Lemma \ref{lem: stability lemma}} \label{sec: proof of Lemma 3}
\begin{proof}
Note that the empirical objective $\hat{\phi}(w, \lambda) = \sum_{i=1}^p \lambda_i L_{S_i}(w)$ can also be written in a finite sum form. For any $i\in [p]$, we denote $z_i^j$ be the $j_{th}$ datapoint of dataset $S_i$ and define batched node $\xi^j = (z_1^j, \dots z_p^j)$. Hence $\tilde{S}$ can be expressed as $\tilde{S} = \{\xi^1, \dots \xi^n\}$. We also define a new loss function $f(w, \lambda, \xi)$ where $\xi = \{z_1, \dots z_p\}$ as
\begin{equation*}
    f(w, \lambda, \xi) = \sum_{i=1}^p \lambda_i \ell(w, z_i)
\end{equation*}

Then it is easy to show that
\begin{equation}
    \hat{\phi}(w, \lambda) = \frac{1}{n} \sum_{j=1}^n f(w, \lambda, \xi^j) \label{eq: finite sume form}
\end{equation}

Suppose $\tilde{S} = \{\xi_1, \dots \xi_i, \dots \xi_n\}$ and $\tilde{S}' = \{\xi_1, \dots \xi'_i, \dots \xi_n\}$ where $\xi_i$ and $\xi'_i$ differ with single datapoint $z_i^j$ for some $j\in [n]$. By equation $\eqref{eq: finite sume form}$, we have
\begin{align*}
    \hat{\phi}(w, \lambda) &= \frac{1}{n} \sum_{\xi \in \tilde{S}} f(w, \lambda, \xi) \\ 
    \hat{\phi}'(w, \lambda) &= \frac{1}{n} \sum_{\xi \in \tilde{S}'} f(w, \lambda, \xi) \\
\end{align*}
We also let $ \Psi(w, \lambda) = \frac{\mu_w}{2}\norm{w-w'}_2^2 - \mu_\lambda \sum_{j=1}^p \lambda_j \log \lambda_j$.

Let $(\hat{w}, \hat{\lambda})$ and $(\hat{w}', \hat{\lambda}')$ be the exact saddle point of $\hat{\phi}(w, \lambda) + \Psi(w, \lambda)$ and $\hat{\phi}'(w, \lambda) +\Psi(w, \lambda)$ respectively.

It is easy to show that $\norm{\nabla_w f(w, \lambda, \xi)}_2\leq L$ and $\norm{\nabla_\lambda f(w, \lambda, \xi)}_\infty \leq B$ for any $\xi$. Also, $\hat{\phi}(w, \lambda) + \Psi(w, \lambda)$ is $\mu_w$-strongly convex w.r.t $\norm{\cdot}_2$ and $\mu_\lambda$-strongly concave w.r.t $\norm{\cdot}_1$. Then using Lemma 2 from \cite{zhang2021generalization}, we obtain
\begin{equation*}
\sqrt{\mu_w\norm{\hat{w} - \hat{w}'}_2^2 + \mu_\lambda \norm{\hat{\lambda}-\hat{\lambda}'}_1^2} \leq \frac{2}{n} \sqrt{\frac{L^2}{\mu_w} + \frac{B^2}{\mu_\lambda}}
\end{equation*}

Furthermore, we have
\begin{align*}
        & \norm{\hat{w}- \hat{w}'}_2 \leq \sqrt{\mu_w\norm{\hat{w}-\hat{w}'}_2^2 + \mu_\lambda \norm{\hat{\lambda}-\hat{\lambda}'}_1^2}/\sqrt{\mu_w} \\
        &\leq \frac{1}{n}\sqrt{\frac{4L^2}{\mu_w^2} + \frac{4B^2}{\mu_\lambda \mu_w}} \leq \frac{2}{n}\left(\frac{L}{\mu_w} + \frac{B}{\sqrt{\mu_w \mu_\lambda}}\right)
\end{align*}

Meanwhile, by Lemma \ref{lem: appximation distance}, we also have 
\begin{equation*}
        \norm{\tilde{w}-\hat{w}}_2 \leq \sqrt{\frac{2\alpha}{\mu_w}} \leq  \frac{1}{2n}\left(\frac{L}{\mu_w} + \frac{B}{\sqrt{\mu_w \mu_\lambda}}\right)
\end{equation*}
and 
\begin{equation*}
    \norm{\tilde{w}'-\hat{w}'}_2 \leq \sqrt{\frac{2\alpha}{\mu_w}} \leq  \frac{1}{2n}\left(\frac{L}{\mu_w} + \frac{B}{\sqrt{\mu_w \mu_\lambda}}\right)
\end{equation*}

Putting every thing together, we have
\begin{equation*}
    \norm{\tilde{w}- \tilde{w}'}_2 \leq  \frac{3}{n}\left(\frac{L}{\mu_w} + \frac{B}{\sqrt{\mu_w \mu_\lambda}}\right)
\end{equation*}
\end{proof}

\subsection{Proof of Lemma \ref{lem: expected excess loss of hat_w}} \label{sec: proof of Lemma 4}
\begin{proof}
Let $R(w) = \max_{i\in [p]} L_{D_i}(w)$ and $\hat{\phi}(w, \lambda) = \sum_{i\in [p]} \lambda_i L_{S_i}(w)$. Denote $(\hat{w}, \hat{\lambda})$ the exact saddle point of $F(w, \lambda)$. By the definition of saddle point, we have for any $w\in \cW$ and $\lambda \in \Delta_p$,
\begin{align*}
    &\hat{\phi}(\hat{w}, \lambda) + \frac{\mu_w}{2} \norm{\hat{w}- w'}_2^2 - \mu_\lambda \sum_{j=1}^p \lambda_j \log \lambda_j \\
    &- \left(\hat{\phi}(w, \hat{\lambda}) +  \frac{\mu_w}{2} \norm{w- w'}_2^2 - \mu_\lambda \sum_{j=1}^p \hat{\lambda}_j \log \hat{\lambda}_j\right)\leq 0 \\ 
    &\implies \hat{\phi}(\hat{w}, \lambda) - \hat{\phi}(w, \hat{\lambda}) \leq \\ 
    &\frac{\mu_w}{2} \norm{w- w'}_2^2 - \mu_\lambda \sum_{j=1}^p \hat{\lambda}_j \log \hat{\lambda}_j - \left( \frac{\mu_w}{2} \norm{\hat{w}- w'}_2^2 - \mu_\lambda \sum_{j=1}^p \lambda_j \log \lambda_j \right) \\
    &\leq \frac{\mu_w}{2} \norm{w- w'}_2^2 - \mu_\lambda \sum_{j=1}^p \hat{\lambda}_j \log \hat{\lambda}_j \\
    &\leq \frac{\mu_w}{2} \norm{w- w'}_2^2 + \mu_\lambda \log p
\end{align*}

In particular, we have for any $w\in \cW$
\begin{equation*}
    \max_{i \in [p]} L_{S_i}(\hat{w}) - \hat{\Phi}(w, \hat{\lambda}) \leq \frac{\mu_w}{2} \norm{w- w'}_2^2 + \mu_\lambda \log p
\end{equation*}

By Lemma \ref{lem: stability lemma}, for any fixed $i\in [p]$, computing $\hat{w}$ has a uniform stability of $\gamma = \frac{2L^2}{n\mu_w} + \frac{2LB}{n\sqrt{\mu_w \mu_\lambda}}$ with respect to the change of datapoint in dataset $S_i$. Therefore, by fixing the randomness of $\tilde{S}_{-i} = \{S_1, \dots S_{i-1}, S_{i+1}, \dots S_{p}\}$ and Theorem 1.1 from \cite{feldman2019high}, we have 
\begin{equation*}
    P_{S_i \sim D_i^n}\left(|L_{S_i}(\hat{w})-L_{D_i}(\hat{w})|\geq c \left( \gamma \log(n)\log(n/\beta) + \frac{B\sqrt{\log 1/\beta}}{\sqrt{n}} \right)\right) \leq \beta
\end{equation*}
Then we can release the randomness of $\tilde{S}_{-i}$ and obtain
\begin{equation*}
    P_{\tilde{S}}\left(|L_{S_i}(\hat{w})-L_{D_i}(\hat{w})|\geq c \left( \gamma \log(n)\log(n/\beta) + \frac{B\sqrt{\log 1/\beta}}{\sqrt{n}} \right)\right) \leq \beta
\end{equation*}

Using union bound across all $i \in [p]$, we obtain with probability over $1-\beta$, for all $i\in [p]$
\begin{equation*}
    |L_{S_i}(\hat{w}) - L_{D_i}(\hat{w})| = O \left(\left(\frac{L^2}{n\mu_w} + \frac{LB}{n\sqrt{\mu_w \mu_\lambda}}\right)\log(n)\log(np/\beta) + \frac{B\sqrt{\log(p/\beta)}}{\sqrt{n}}\right)  
\end{equation*}

Therefore, with probability $1-\beta$, we have 

\begin{align*}
    &|\max_{i \in [p]} L_{D_i}(\hat{w}) - \max_{i\in [p]} L_{S_i}(\hat{w})|\\
    & = O \left(\left(\frac{L^2}{n\mu_w} + \frac{LB}{n\sqrt{\mu_w \mu_\lambda}}\right)\log(n)\log(np/\beta) + \frac{B\sqrt{\log(p/\beta)}}{\sqrt{n}}\right)  
\end{align*}

Also, since $w$ is independent of the dataset, with probability over $1-\beta$ 
\begin{equation*}
    |L_{S_i}(w) - L_{D_i}(w)| = O\left( \frac{B\log(p/\beta)}{\sqrt{n}}\right) \quad \forall i\in [p]
\end{equation*}
and 
\begin{equation*}
    |\Phi(w,\hat{\lambda}) - \hat{\Phi}(w, \hat{\lambda})| = O\left( \frac{B\log(p/\beta)}{\sqrt{n}}\right) 
\end{equation*}

After putting everything together, we obtain with probability over $1-2\beta$
\begin{align*}
    &R(\hat{w}) - R(w) = \max_{i\in [p]} L_{D_i}(\hat{w}) - \max_{i \in [p]} L_{D_i}(w) \\
    &\leq \max_{i\in [p]} L_{D_i}(\hat{w}) - \phi(w, \hat{\lambda}) \\
    &\leq \max_{i\in [p]} L_{S_i}(\hat{w}) - \hat{\phi}(w, \hat{\lambda}) + O \left(\left(\frac{L^2}{n\mu_w} + \frac{LB}{n\sqrt{\mu_w \mu_\lambda}}\right)\log(n)\log(np/\beta) + \frac{B\sqrt{\log(p/\beta)}}{\sqrt{n}}\right)  \\
    &= O \left(\frac{\mu_w}{2} \norm{w- w'}_2^2 + \mu_\lambda \log p + \left(\frac{L^2}{n\mu_w} + \frac{LB}{n\sqrt{\mu_w \mu_\lambda}}\right)\log(n)\log(np/\beta) + \frac{B\sqrt{\log(p/\beta)}}{\sqrt{n}}\right)
\end{align*}

Choosing $\beta=1/n$ and taking expectation over the dataset collection $\tilde{S}$, we obtain
\begin{align}
    &\E\left[R (\hat{w})\right] - R(w) \label{eq: excess risk of hat_w}\\
    &= O \left(\frac{\mu_w}{2} \norm{w- w'}_2^2 + \mu_\lambda \log p + \left(\frac{L^2}{n\mu_w} + \frac{LB}{n\sqrt{\mu_w \mu_\lambda}}\right)\log(n)\log(np) + \frac{B\sqrt{\log(pn)}}{\sqrt{n}}\right) \nonumber
\end{align}

By Lemma \ref{lem: appximation distance}, we have
\begin{equation*}
        \norm{\tilde{w}-\hat{w}}_2 \leq \sqrt{\frac{2\alpha}{\mu_w}} \leq  \frac{1}{2n}\left(\frac{L}{\mu_w} + \frac{B}{\sqrt{\mu_w \mu_\lambda}}\right)
\end{equation*}

Given the fact that $R(\cdot)$ is $L$-Lipschitz, we have 
\begin{equation}
    |R(\tilde{w}) - R(\hat{w})|\leq L \norm{\tilde{w}-\hat{w}}_2 = O \left(\frac{L^2}{n\mu_w} + \frac{LB}{n\sqrt{\mu_w \mu_\lambda}}\right) \label{eq: risk diff between hat_w and tilde_w}
\end{equation}

Combining equations \eqref{eq: excess risk of hat_w} and \eqref{eq: risk diff between hat_w and tilde_w}, we have
\begin{align*}
    &\E\left[R (\tilde{w})\right] - R(w) \nonumber\\
    &= O \left(\frac{\mu_w}{2} \norm{w- w'}_2^2 + \mu_\lambda \log p + \left(\frac{L^2}{n\mu_w} + \frac{LB}{n\sqrt{\mu_w \mu_\lambda}}\right)\log(n)\log(np) + \frac{B\sqrt{\log(pn)}}{\sqrt{n}}\right) 
\end{align*}

\end{proof}

\subsection{Proof of Theorem \ref{thm: expected worst-group excess risk}} \label{sec:proof of Theorem 2}
\begin{proof}
\textbf{Privacy:} At iteration $t$, by the stability argument in Lemma \ref{lem: stability lemma}, we have
\begin{align*}
    \norm{\tilde{w_t} - \tilde{w_t}'}_2 &\leq \frac{3L}{n_t \mu_w^t} + \frac{3B}{n_t \sqrt{\mu_w^t \mu_\lambda^t}} \\
    &= 3L \eta_t + \frac{3B \sqrt{n_t n \eta_t \eta}}{n_t} \\
    &\leq 6D \sqrt{(\log_2 n) \eta_t \eta}
\end{align*}
where $\tilde{w_t}$ and $\tilde{w_t}'$ are outputs from neighboring dataset collections $\tilde{S}^t$ and $\tilde{S}^{t'}$ %\rnote{This looks weird. Write it as $\tilde{S}^{t'}$} 
that differ in one datapoint. The privacy follows from the privacy guarantee of Gaussian mechanism.

\textbf{Utility:} By Lemma \ref{lem: expected excess loss of hat_w}, we have
\begin{align*}
    &\E[R (\tilde{w_t}) - R (\tilde{w}_{t-1})] \\
    &=O\left( \frac{\E[\norm{\xi_{t-1}}_2^2]}{n_t \eta_t}  + \frac{\log p}{\eta n} + D^2\sqrt{\eta \eta_t}\log^{1.5} n \log (np) + \frac{B\sqrt{\log (pn)}}{\sqrt{n_t}} \right)
\end{align*}
where $R(w) = \max_{i\in [p]} L_{D_i}(w)$.

Also, we have 
\begin{equation*}
    \E \norm{\xi_t}_2^2 = d\sigma_t^2 = 72d D^2 \log (n) \log (1/\delta) 2^{-t} \eta^2 /\epsilon^2
\end{equation*}
Therefore, as long as
\begin{equation*}
    \eta \leq \frac{M\epsilon}{D \sqrt{72d \log (n) \log(1/\delta)}}
\end{equation*}
We will have 
\begin{equation*}
    \E \norm{\xi_t}_2^2 \leq 2^{-t} M^2 \text{ and } \E \norm{\xi_t}_2 \leq \sqrt{2}^{-t} M
\end{equation*}

Let $\tilde{w}_0 = w^*$ and $\xi_0 = w_0 - w^*$. Then $\norm{\xi_0}_2\leq M$. Hence
\begin{align*}
    &\E [R(w_T)] - R(w^*) = \sum_{t=1}^T  \E [R(\tilde{w}_t) - R(\tilde{w}_{t-1})] + \E [R(w_T) - R(\hat{w}_{T})] \\
    &\leq \sum_{t=1}^T O\left( \frac{\E[\norm{\xi_{t-1}}_2^2]}{n_t \eta_t}  + \frac{\log p}{\eta n} + D^2\sqrt{\eta \eta_t}\log^{1.5} n \log (np) + \frac{B \sqrt{\log (n)\log (pn)}}{\sqrt{n}} \right) \\
    &\quad + L \E[\norm{\xi_T}_2] \quad \text{($R(\cdot)$ is $L$-Lipschitz.)}
\end{align*}

We have 
\begin{equation*}
    \frac{\E[\norm{\xi_{t-1}}_2^2]}{n_t \eta_t} \leq \frac{2^{-(t-1)}M^2\log_2 n}{n 2^{-t} \eta} = \frac{2(\log_2 n) M^2}{\eta n}
\end{equation*}
and $\E[\norm{\xi_T}_2] \leq  M \sqrt{2}^{-\log_2 n} = \frac{M}{\sqrt{n}}$. 

Therefore,
\begin{align*}
    &\E [R(w_T)] - R(w^*) \\
    &\leq \sum_{t=1}^T O\left( \frac{\log (n) M^2}{\eta n}  + \frac{\log (p)}{\eta n} + D^2\eta \log^{1.5} (n) \log (np) + \frac{B \sqrt{\log (n)\log (pn)}}{\sqrt{n}} \right) \\ 
    &\quad + \frac{ML}{\sqrt{n}}\\
    & =\log n \cdot O\left( \frac{(\log (np)) M^2}{\eta n}  + D^2\eta \log^{2.5} (np)+ \frac{B \sqrt{\log (n)\log (np)}}{\sqrt{n}} \right) \\
    &\quad + \frac{ML}{\sqrt{n}}  \\
    &=O\left(\frac{MD \log^{11/4} (np)}{\sqrt{n}} + (MD \log^{5/2} (np) ) \frac{\sqrt{d \log (1/\delta)}}{n\epsilon}\right)
\end{align*}
Replacing $n$ with $K/p$ gives the desired result.
\end{proof}

\section{Missing Proofs in Section \ref{sec: DP OCO}} 

\subsection{Proof of Theorem \ref{thm: utility guarantee online learning}} \label{sec:proof of DP OCO}

\begin{proof}
It is easy to see that the objective in \eqref{eq:agnostic loss} is equivalent to 
    \begin{equation*}
        \min_{w\in \cW} \max_{\lambda \in \Delta_p} \sum_{i\in [p]} \lambda_i L_{D_i}(w)
    \end{equation*}
    where $\Delta_p = \{\lambda \in [0, 1]^p : \norm{\lambda}_1 = 1\}$ the probability simplex over $p$ users. Recall $\phi(w, \lambda) = \sum_{i\in [p]} \lambda_i L_{D_i}(w)$, then we have
    \begin{align*}
        &\max_{\lambda \in \Delta_p} \phi(\bar{w}_T, \lambda) - \min_{w\in \cW} \max_{\lambda \in \Delta_p}\phi(w, \lambda)\\
        &= \max_{\lambda \in \Delta_p} \phi(\bar{w}_T, \lambda) - \max_{\lambda \in \Delta_p}\min_{w\in \cW} \phi(w, \lambda) \\
        &\leq \frac{1}{T}\max_{w\in \cW, \lambda \in \Delta_p} \left\{ \sum_{t=1}^T \phi(w_t, \lambda) - \phi(w, \lambda_t) \right\} \\
        &=\frac{1}{T}\max_{w\in \cW, \lambda \in \Delta_p} \left\{ \sum_{t=1}^T \left(\phi(w_t, \lambda) - \phi(w_t, \lambda_t) + \phi(w_t, \lambda_t) - \phi(w, \lambda_t)\right) \right\} \\
        &= \frac{1}{T} \left[\sum_{t=1}^T \phi(w_t, \lambda_t) - \min_{w\in \cW} \sum_{t=1}^T\phi(w, \lambda_t)\right] \quad (A) \\
        &\quad + \frac{1}{T} \left[\sum_{t=1}^T  -\phi(w_t, \lambda_t) - \min_{\lambda\in \Delta_p} \sum_{t=1}^T (-\phi(w_t, \lambda)) \right] \quad (B)
    \end{align*}
    We address the term A first. By fixing the randomness of $w_t$ and $\lambda_t$ for some iteration $t$, we have $\bE_{x_{t}^-} \ell(w_t, x_{t}^-) = \phi(w_t, \lambda_t)$. By extending the same argument to all $t\in [T]$, we obtain for any $w$,
    \begin{equation*}
        \frac{1}{T}\bE\left[\sum_{t=1}^T \left(\phi(w_t, \lambda_t) -\phi(w, \lambda_t)\right)\right] = \frac{1}{T}\bE\left[\sum_{t=1}^T \ell(w_t, x_{t}^-) - \ell(w, x_{t}^-)\right]
    \end{equation*}

    By the regret guarantee of $\cQ_-$, we have for any sequence $x_{1}^-, \dots x_{T}^-$
    \begin{equation*}
        \frac{1}{T}\bE \left[\sum_{t=1}^T \ell(w_t, x_{t}^-) - \ell(w, x_{t}^-)\right] \leq r_{Q_-}(T)
    \end{equation*}
    where the expectation is taken from the randomness over the algorithm. Therefore, we have 
    \begin{equation}
        \bE [A] \leq r_{Q_-}(T) \label{eq:bound for term A}
    \end{equation}

    We then address the term $B$. We first define $\ell'(w, x)$ as the difference between some fixed constant $U$ and the original loss function $\ell(w, x)$, that is, $\ell'(w, x) = U -\ell(w, x)$. Similarly  the vector of $p$ population risks over the new loss function $\ell'$ is therefore written as $g(w) = [L'_{D_1}(w),\dots L'_{D_p}(w)]$. Therefore, the term $B$ can be equivalently written as
    \begin{equation*}
        B = \frac{1}{T} \left[\sum_{t=1}^T \langle g(w_t), \lambda_t \rangle - \min_{\lambda \in \Delta^p} \sum_{t=1}^T \langle g(w_t), \lambda \rangle \right]
    \end{equation*}

    Notice that at each iteration $t$, we have 
    \begin{equation*}
        \nabla_\lambda \langle g(w_t), \lambda \rangle =  g(w_t)
    \end{equation*}

    We construct a unbiased estimator of $g(w_t)$ as
    \begin{equation*}
        g'(w_t) = \begin{cases}
        \frac{\Tilde{\ell}_t}{\lambda_{t}^{i}} \quad i = j_t \\
        0 \quad o.w.
        \end{cases}
    \end{equation*}

    By letting $U = B + \frac{2B}{\epsilon}\log(T)$, we have we have w.p. over $1-\frac{1}{T}$,  $\max_{t\in [T]} |y_t| \leq \frac{2B}{\epsilon} \log(T)$, which we denote as event $E$. Conditioned on $E$, we have $\Tilde{l}_t \in [0, 2U]$ for any $t\in [T]$.
    
    Observing that for any fixed $w_t$, since $x_t^+ \sim D_{i_t}$, we obtain
    \begin{equation*}
        \bE [g'(w_t)|E] = g(w_t)
    \end{equation*}

    We now establish the bound on term B
    \begin{align*}
        &\bE \left[\sum_{t=1}^T \langle g(w_t), \lambda_t \rangle - \sum_{t=1}^T \langle g(w_t), \lambda \rangle \right]\\
        &= \bE \left[ \sum_{t=1}^T \langle g'(w_t), \lambda_t \rangle - \sum_{t=1}^T \langle g'(w_t), \lambda \rangle \right]
    \end{align*}

    Since $\norm{g'(w_t)}_{\infty}\leq 2U$ conditioned on event $E$, by setting $\eta = \sqrt{\frac{\ln(p)}{pTU^2}}$ and invoking the regret bound of EXP3, we obtain
    \begin{align*}
        \bE[B|E] &= \frac{1}{T}\bE \left[\sum_{t=1}^T \langle g(w_t), \lambda_t \rangle - \sum_{t=1}^T \langle g(w_t), \lambda \rangle \right] \\
        &= O\left(U\sqrt{\frac{p\log(p)}{T}}\right)
    \end{align*}

    Then we take the expectation on $E$ and obtain
    \begin{equation*}
        \bE [B] = P(E) \bE[B|E] + (1-P(E)) \bE[B|E^c] = O\left(U\sqrt{\frac{p\log(p)}{T}}\right)
    \end{equation*}

    Finally we plug in the value of $U$,
    \begin{equation}
        \bE [B] = O\left(\frac{B\log(T)}{\epsilon}\sqrt{\frac{p\log(p)}{T}}\right) \label{eq:bound for term B}
    \end{equation}

    We combine equations \eqref{eq:bound for term A} and \eqref{eq:bound for term B} to get the desired result.
\end{proof}

\section{Missing Proofs from Section \ref{sec:offline setting}} \label{sec:proof in offline setting}

\subsection{Proof of Theorem \ref{thm: privacy gurantee of Noisy-SGD-MGR} and \ref{thm:convergence rate Noisy-SGD-MGR}} \label{sec: proof of noisy-SGD-MGR}

\begin{proof} (Proofs of Theorem \ref{thm: privacy gurantee of Noisy-SGD-MGR})
     Note that given two neighboring data collections $\tilde{S} = \{S_1, \dots S_p\}$ and $\tilde{S}' = \{S'_1, \dots S'_p\}$ that differ in one datapoint of $j_{th}$ dataset for some $j\in [p]$. In iteration $t$, if $i_t \neq j$, then the output distribution of $w_{t+1}$ is same for $\tilde{S}$ and $\tilde{S}'$. Otherwise, generating $w_{t+1}$ satisfies $(\frac{\epsilon}{c \sqrt{T\log(1/\delta)}}, \frac{\delta}{2T})$-DP based on the property of Gaussian mechanism and privacy amplification by subsampling. Therefore, overall generating $w_{t+1}$ is $(\frac{\epsilon}{c \sqrt{T\log(1/\delta)}}, \frac{\delta}{2T})$-DP.

    Meanwhile, generating $\lambda_{t+1}$ is $(\frac{\epsilon}{c \sqrt{T\log(1/\delta)}}, 0)$-DP based on the property of Laplace mechanism. Using basic composition, we have generating the pair $(w_{t+1}, \lambda_{t+1})$ is $(\frac{2\epsilon}{c \sqrt{T\log(1/\delta)}}, \frac{\delta}{2T})$-DP. Hence, by strong composition for $t\in [T]$, we obtain the algorithm is $(\epsilon, \delta)$-DP.
\end{proof}

\begin{proof} (Proof of Theorem \ref{thm:convergence rate Noisy-SGD-MGR})
Recall that $n= K/p$. We have
    \begin{align*}
        &\max_{\lambda \in \Delta_p} \hat{\phi}(\bar{w}_T, \lambda) - \min_{w\in \cW} \max_{\lambda \in \Delta_p}\hat{\phi}(w, \lambda)\\
        &\leq \max_{\lambda \in \Delta_p} \hat{\phi}(\bar{w}_T, \lambda) - \min_{w\in \cW} \hat{\phi}(w, \bar{\lambda}_T) \\
        &\leq \max_{\lambda\in \Delta_p, w\in \cW} \left\{ \frac{1}{T} \sum_{t=1}^T \left(\hat{\phi}(w_t, \lambda) - \hat{\phi}(w, \lambda_t)\right)  \right\} \\
        & = \max_{\lambda\in \Delta_p, w\in \cW} \left\{ \frac{1}{T} \sum_{t=1}^T \left(\hat{\phi}(w_t, \lambda) + \hat{\phi}(w_t, \lambda_t) -  \hat{\phi}(w_t, \lambda_t)  - \hat{\phi}(w, \lambda_t)\right)  \right\} \\
        &=  \frac{1}{T} \left[\sum_{t=1}^T \hat{\phi}(w_t, \lambda_t) - \min_{w\in \cW} \sum_{t=1}^T \hat{\phi}(w, \lambda_t)\right] \quad (A) \\
        &\quad + \frac{1}{T} \left[\sum_{t=1}^T  - \hat{\phi}(w_t, \lambda_t) - \min_{\lambda\in \Delta_p} \sum_{t=1}^T (-\hat{\phi}(w_t, \lambda)) \right] \quad (B)
    \end{align*}

Next, we will bound the terms $A$ and $B$ separately. We address term $A$ first. In particular, for any $w\in \cW$, we can rewrite the expectation of term $A$ as
\begin{align*}
    &\E[A] = \frac{1}{T} \E\left[\sum_{t=1}^T \left(\hat{\phi}(w_t, \lambda_t) -  \hat{\phi}(w, \lambda_t)\right)\right] \\
    &\leq \frac{1}{T}\E\left[\sum_{t=1}^T \langle \nabla \hat{\phi}(w_t, \lambda_t), w_t - w\rangle \right]
\end{align*}

For any iteration $t$, by fixing the randomness up to $w_t$ and $\lambda_t$ and letting $\nabla_t = \frac{1}{m}\sum_{z\in B_t} \nabla\ell(w_t, z) +G_t$ be the gradient update in step 6 in Algorithm \ref{alg:Noisy-SGD-MW}, we have
\begin{equation*}
    \E [\nabla_t] =  \nabla \hat{\phi}(w_t, \lambda_t)
\end{equation*}
We can extend the same augment for all $t\in [T]$ and obtain
\begin{align*}
    & \E[A] \leq \frac{1}{T}\E\left[\sum_{t=1}^T \langle \nabla \hat{\phi}(w_t, \lambda_t), w_t - w\rangle \right] \\
    &= \frac{1}{T}\E\left[\sum_{t=1}^T \langle \nabla_t, w_t - w\rangle \right] \\
    &\leq \frac{M^2}{2T\eta} + \frac{\eta}{2T}\sum_{t=1}^T \E\left[\norm{\nabla_t}^2\right] \\
    &\leq \frac{\eta L^2}{2} + \frac{M^2}{2T\eta} + \frac{d\sigma^2\eta}{2}
\end{align*}
where the first equality holds by standard convex analysis.

Meanwhile, we let $U  = B + \frac{8B}{n\epsilon}\sqrt{T\log(1/\delta)}\log(pnT/2)$, and denote $L'_t = [U-L_{S_i}(w_t)+y_{i, t}]_{i=1}^p$. By using a union bound, we have $\norm{L'(t)}_\infty\leq 2U$ for all $t\in [T]$ with probability over $1-\frac{1}{n}$, which we denote as event E. Conditioned on event E, our update rule for $\lambda$ can be seen as applying the Hedge algorithm with $L'_t$ as the loss input and the term $B$ is formulated as its corresponding regret bound. Therefore, by the regret bound of Hedge , we obtain
\begin{equation*}
    \E[B|E] = O \left( \norm{L_t}_\infty \sqrt{\frac{\log (p)}{T}}  \right) = O \left(  B\sqrt{\frac{\log (p)}{T}} +  \frac{B\sqrt{\log (p)\log(1/\delta)}\log(pnT)}{n\epsilon} \right)
\end{equation*}
Then we can take the expectation over $E$ and have
\begin{equation*}
    \E [B] = P(E) \E[B|E] + P(E^c) \E[B|E^c] = O \left(  B\sqrt{\frac{\log (p)}{T}} +  \frac{B\sqrt{\log (p)\log(1/\delta)}\log(pnT)}{n\epsilon} \right) 
\end{equation*}
By combining the bounds of $\E[A]$ and $\E[B]$ and plugging into the values of $T$, $\eta_w$ and $\eta_\lambda$, we obtain the desired result.
\end{proof}

\subsection{Private Active Group Selection}\label{sec:DP AGS}
Here we present another algorithm built upon the active group selection scheme described in Algorithm \ref{alg:Noisy-SGD-AGS}. Algorithm \ref{alg:Noisy-SGD-AGS} is different from Algorithm \ref{alg:Noisy-SGD-MW} on the group selection method. Instead of maintaining a weight vector to sample the dataset as in Algorithm \ref{alg:Noisy-SGD-MW}, we select the dataset with highest loss in each iteration. To maintain privacy, we resort to the report-noisy-max mechanism for the group selection. After such (approximately) worst-off group is selected, we sample a mini-batch from it and use Noisy-SGD to update the model parameters.

\begin{algorithm}
\caption{Noisy SGD with active group selection (Noisy-SGD-AGS)}
\label{alg:Noisy-SGD-AGS}
\begin{algorithmic}
\State {\bfseries Input}Collection of datasets $S = \{S_1, \dots S_p\}\in \cZ^{n\times p}$, mini-batch size $m$, $\#$ iterations $T$, learning rate $\eta$
\For{$t = 1,\dots, T-1$}
    \State Compute $i_t = \argmax_{i\in [p]} L_{S_i}(w_t) + y_{i, t}$ where $y_{i,t}\stackrel{\text{iid}}{\sim}\text{Lap}(\tau)$. 
    \State Sample $B_t = \{z_1,\dots z_m\}$ from $S_{i_t}$ uniformly with replacement. \\
    \State Update the model by 
    \begin{equation*}
        w_{t+1}
        = \text{Proj}_{\cW}\left(w_t - \eta \cdot \left(\frac{1}{m}\sum_{z\in B_t} \nabla\ell(w_t, z) +G_t \right) \right).
    \end{equation*}
    \State where $G_t \sim \mathcal{N}(0, \sigma^2 I_d)$.
\EndFor 

\State {\bfseries Output} $\Bar{w}_T = \frac{1}{T} \sum_{t=1}^T w_t$
\end{algorithmic}
\end{algorithm}

\begin{thm} (Privacy guarantee) \label{sec:privacy guarantee Noisy-SGD-AGS}
Algorithm \ref{alg:Noisy-SGD-AGS} is $(\epsilon, \delta)$-DP with $\tau = \frac{cBp}{K\epsilon}\sqrt{T\log(1/\delta)}$ and $\sigma^2 = \frac{cTL^2q^2\log(K/\delta)\log(1/\delta)}{K^2 \epsilon^2}$ for some universal constant $c$.
\end{thm}.

\begin{thm} \label{thm:convergence rate of Noisy-SGD-AGS}
(Convergence rate) with probability over $1-\kappa$, by letting $T = \frac{MLK\epsilon}{16 Bp \sqrt{\log(1/\delta)}}$ and $\eta=\frac{M}{\sqrt{T(G^2 + d\sigma^2)}}$, we have
\begin{align*}
    &\E [\max_{i\in [p]} L_{S_i}(\Bar{w}_T)] - \min_{w\in \mathcal{W}} \max_{i\in [p]} L_{S_i}(w) = \\
    &O\left(\left(\sqrt{\frac{MLBp}{K\epsilon}} +\frac{MLp\sqrt{d}}{K\epsilon}\right)\cdot \polylog(p, \delta, \kappa, n)\right).
\end{align*}
\end{thm}

\begin{proof}
By the union bound, we have with probability over $1-\kappa$ 
\begin{equation*}
    |y_{i, t}| \leq \frac{8B}{n\epsilon} \sqrt{T\log(1/\delta)} \log(pT/2\kappa)
\end{equation*}
for all $i\in [p]$ and $t\in [T]$. We denote this as event $A$ and condition on $A$. By denoting $R(\kappa, T)=\frac{8B}{n\epsilon} \sqrt{T\log(1/\delta)} \log(pT/2\kappa)$, we have
\begin{align*}
    &\E[\max_{i\in [q]} L_{S_i}(\Bar{w}_T)]  \\
    &\leq \E \left[\frac{1}{T}\sum_{t=1}^T \max_{i\in [q]} L_{S_i}(w_t)\right] \quad (\text{convexity of } \max_i L_{S_i}) \\
    &\leq \E \left[\frac{1}{T}\sum_{t=1}^T  L_{S_{i_t}}(w_t)\right] + R(\kappa, T) \quad (\text{Event A})
\end{align*}

For any fixed $w^*$, by the convexity of the loss function, we obtain
\begin{equation*}
    \E\left[\frac{1}{T}\sum_{t=1}^T  (L_{S_{i_t}}(w_t) -  L_{S_{i_t}}(w^*))\right] \leq \frac{1}{T}\E\left[\sum_{t=1}^T \langle \nabla L_{S_{i_t}}(w_t), w_t-w^* \rangle \right]
\end{equation*}

Denote $\nabla_t = \frac{1}{m}\sum_{z\in B_t} \nabla\ell(w_t, z) +G_t$ where $B_t$ is the mini-batch uniformly sampled from $S_{i_t}$ and $G_t \sim \mathcal{N}(0, \sigma^2 I_d)$. If we fix the randomness of $w_t$ and $i_t$, it is easy to see that
\begin{equation*}
    \E [\nabla_t] = \nabla L_{S_{i_t}}(w_t)
\end{equation*}
where the randomness comes from the datapoint sampling and added Gaussian noise. 

By releasing the randomness of $w_t$ and $i_t$ and extending the same analysis to all $t\in [T]$, we obtain
\begin{align*}
    \E\left[\frac{1}{T}\sum_{t=1}^T  \left(L_{S_{i_t}}(w_t) -  L_{S_{i_t}}(w^*)\right)\right]&\leq \frac{1}{T}\E\left[\sum_{t=1}^T \langle \nabla L_{S_{i_t}}(w_t), w_t-w^* \rangle \right] \\
    &=\frac{1}{T}\E\left[\sum_{t=1}^T \langle \nabla_t, w_t-w^* \rangle \right]\\
    &\leq \frac{M^2}{2T\eta} + \frac{\eta}{2T}\sum_{t=1}^T \E\left[\norm{\nabla_t}^2\right] \\
    &\leq \frac{\eta L^2}{2} + \frac{M^2}{2T\eta} + \frac{d\sigma^2\eta}{2}
\end{align*}

Therefore, we have 
\begin{align*}
    &\E[\max_{i\in [q]} L_{S_i}(\Bar{w}_T)]  \\
    &\leq \E \left[\frac{1}{T}\sum_{t=1}^T  L_{S_{i_t}}(w_t)\right] + R(\kappa, T) \quad (\text{Event A}) \\
    & \leq \E\left[\frac{1}{T}\sum_{t=1}^T  L_{S_{i_t}}(w^*)\right] + \frac{\eta L^2}{2} + \frac{M^2}{2T\eta} + \frac{d\sigma^2\eta}{2} + R(\kappa, T) \\
    &\leq \max_{i\in [q]} L_{S_i}(w^*) + \frac{\eta L^2}{2} + \frac{M^2}{2T\eta} + \frac{d\sigma^2\eta}{2} + R(\kappa, T)
\end{align*}

By plugging into the value of $\eta$, $\sigma^2$ and $T$ and replacing $n$ with $K/p$, we get the desired result.
\end{proof}

\section{Lower Bounds for the Offline Setting } \label{sec:lower bound instances}
%\rnote{You should go beyond expressing the instance. You should explain how the lower bound is obtained (i.e., the reduction from DP ERM and DP SCO). Also, I don't see any reference to the known lower bounds for DP ERM or DP SCO and citing the prior works!!}\xnote{I added the reduction to DP-ERM and DP-SCO and cited the prior works.}
Consider a $L$-Lipschitz convex loss function $\ell(w, z)$ bounded by $[0, B]$ for any $w\in \cW$ and $z\in \cZ$. We create a new loss function $\ell'(w, (z, y)) = \ell(w, z) + y$ where $y\in [0,B]$. It is easy to see that $\ell'(\cdot, (z, y))$ is also convex and $L$-Lipschitz.

\paragraph{Empirical case}
We let $y=B$ for all datapoints in $S_1 = \{(z_1, y_1), \dots (z_n, y_n)\}$ and $y=0$ for all datapoints in $S_i$ with $i\neq 1$. Therefore, it is easy to see that for any $w\in \cW$ and $i\in [p]$ %\rnote{$q$ should be $p$}
\begin{equation*}
    L'_{S_1}(w) \geq L'_{S_i}(w)
\end{equation*}
where $L'_{S_i}(w) = \frac{1}{n} \sum_{(z, y)\in S_i} \ell'(w, (z, y))$.

Then we have
\begin{equation*}
    \min_{w\in \cW} \max_{i\in [p]} L'_{S_1}(w) = \min_{w\in \cW} L'_{S_1}(w)
\end{equation*}
%\rnote{$q$ should be $p$ under the max above}

Then the original worst-group empirical risk minimization problem is reduced to single-distribution DP-ERM problem which is known to be lower bounded by $\Omega\left(\frac{\sqrt{d}}{n\epsilon}\right)$ for Lipschitz convex loss \cite{bassily2014private}. Therefore, a lower bound of the excess worst-group empirical risk is also $\Omega\left(\frac{\sqrt{d}}{n\epsilon}\right)$. Since $n=K/p$, this lower bound can also be written as $\Omega\left(\frac{p\sqrt{d}}{K\epsilon}\right)$.

\paragraph{Population case}
We let $y=B$ when $(z, y)\sim D_1$ and $y=0$ when $(z, y)\sim D_i$ with $i\neq 1$. The distribution of $z$ is arbitrary. Therefore, we have for any $w\in \cW$ and $i\in [p]$
\begin{equation*}
    L'_{D_1}(w) \geq L'_{D_i}(w)
\end{equation*}
and 
\begin{equation*}
    \min_{w\in \cW} \max_{i\in [p]} L'_{D_i}(w) = \min_{w\in \cW} L'_{D_1}(w)
\end{equation*}
%\rnote{$q$ should be $p$ under the max above}

Similar to the empirical case, we can reduce the worst-group population risk minimization problem to single-distribution DP-SCO problem which is lower bounded by $\Omega\left(\frac{\sqrt{d}}{n\epsilon} + \frac{1}{\sqrt{n}}\right)$ \cite{bassily2019private}. Therefore a lower bound of the population excess worst-group risk is also $\Omega\left(\frac{\sqrt{d}}{n\epsilon} + \frac{1}{\sqrt{n}}\right)$. Since $n=K/p$, this lower bound can be also written as $\Omega\left(\frac{p\sqrt{d}}{K\epsilon} + \sqrt{\frac{p}{K}}\right)$. %\rnote{$q$ should be $p$ in the last expression!}

\end{document}